\documentclass{article}
\pdfoutput=1

\usepackage[final,nonatbib]{neurips_2019}     
\usepackage[utf8]{inputenc} 
\usepackage[T1]{fontenc}    
\usepackage{hyperref}       
\usepackage{url}            
\usepackage{booktabs}       
\usepackage{amsfonts}       
\usepackage{nicefrac}       
\usepackage{microtype}      
\usepackage{diagbox} 
\usepackage{multirow}

\title{Learning Deterministic Weighted Automata \\ with Queries and Counterexamples}

\author{%
  Gail Weiss \\
  Technion\\
  \texttt{sgailw@cs.technion.ac.il} \\
  \And
  Yoav Goldberg \\
  Bar Ilan University \\
  Allen Institute for AI\\
  \texttt{yogo@cs.biu.ac.il} \\
  \And
  Eran Yahav\\
  Technion \\
  \texttt{yahave@cs.technion.ac.il} \\
}

\usepackage{mathtools}
\usepackage{amssymb} 
\usepackage{paralist} 
\usepackage{amsthm} 

\usepackage{subfig}
\usepackage{cleveref}

\usepackage{ifthen}

\newcommand{\DONE}[1]{}

\newtheorem{theorem}{Theorem}[section]
\newtheorem{lemma}[theorem]{Lemma}

\newtheorem{corollary}[theorem]{Corollary}

\crefname{lemma}{lemma}{lemmas} 
\newcommand{\ignore}[1]{}

\newcounter{programlinenumber}

\newboolean{TR}
\setboolean{TR}{false}
\ifthenelse{\boolean{TR}}{

\newcommand{\TrOnly}[1]{#1}
\newcommand{\SubOnly}[1]{}
\newcommand{\TrOnlyInFootnote}[1]{#1}
\newcommand{\TrOnlyInTable}[1]{#1}}
{

\newcommand{\TrOnly}[1]{}
\newcommand{\SubOnly}[1]{#1}
\newcommand{\TrOnlyInFootnote}[1]{}
\newcommand{\TrOnlyInTable}[1]{}}

\newcommand{\argmax}{\arg\!\max} 
\DeclarePairedDelimiter{\ceil}{\lceil}{\rceil} 
\newcommand{\flexfringe}{\textsc{flexfringe} }
\newcommand{\cexs}[1]{}
\newcommand{\emptystr}{\varepsilon}
\newcommand{\toltype}{variation } 
\newcommand{\toltypeng}{variation}
\newcommand{\uhl}{UHL}
\newcommand{\tomita}{Tomita}
\newcommand{\spice}{SPiCe}

\newcommand{\lstar}{L$^*$}
\newcommand{\lstarnogap}{{L$^*$}}

\newcommand{\concat}{{\cdot}}

\newcommand{\oracle}{\mathcal{O}}
\newcommand{\stopsym}{\$}

\newcommand{\hatdelta}{\hat{\delta}}

\newcommand{\approxbar}[1]{\approx_{#1}} 
\newcommand{\napproxbar}[1]{\not\approx_{#1}} 

\newcommand{\tapprox}{\approxbar{t}}
\newcommand{\ntapprox}{\napproxbar{t}} 

\newcommand{\hypoaut}{\mathcal{A}}

\newcommand{\hiddentext}[1]{}

\newcommand{\para}[1]{\vspace{3pt}\noindent\textbf{\textit{#1}}}

\begin{document}

\maketitle
\begin{abstract}
We present an algorithm for extraction of a probabilistic \emph{deterministic} finite automaton (PDFA) from a given black-box language model, such as a recurrent neural network (RNN). 
The algorithm is a variant of the exact-learning algorithm \lstarnogap, adapted to 
a probabilistic setting with noise.
The key insight 
is the use of conditional probabilities 
for
observations,
and the introduction of a local
tolerance when comparing them. 
When applied to RNNs, our algorithm often achieves better word error rate (WER) and normalised distributed cumulative gain (NDCG) than that achieved by
spectral extraction of weighted finite automata (WFA) from the same networks. 
PDFAs are substantially more expressive than n-grams, and are guaranteed to be stochastic and deterministic -- unlike spectrally extracted WFAs.
\end{abstract}

\section{Introduction}\label{Se:Intro}

We address the problem of learning a probabilistic deterministic finite automaton (PDFA) from a trained recurrent neural network (RNN)~\cite{Elman}. 
RNNs, and in particular their gated variants GRU~\cite{GRU1,GRU2} and LSTM~\cite{LSTM}, are well known to be very powerful for sequence modelling, but are not interpretable.
PDFAs, which explicitly list their states, transitions, and weights, are more interpretable than RNNs~\cite{PDFAsInterpretable}, while still being analogous to them in behaviour: 
both emit a single next-token distribution from each state, and have deterministic state transitions given a state and token.
They are also much faster to use than RNNs, as their sequence processing does not require matrix operations.

We present an algorithm for reconstructing a PDFA from any given black-box distribution over sequences, such as an RNN trained with a language modelling objective (LM-RNN). The algorithm is applicable for reconstruction of any weighted deterministic finite automaton (WDFA), and is guaranteed to return a PDFA when the target is stochastic -- as an LM-RNN is.

\para{Weighted Finite Automata (WFA)}
A WFA is a weighted \emph{non-deterministic} finite automaton, capable of encoding language models but also other, non-stochastic weighted functions.
Ayache et al.~\cite{WFAfromRNNSpectral} and Okudono et al.~\cite{WFAfromRNNRegression} show how to apply \emph{spectral learning}~\cite{SpectralLearningofWFAs} to an LM-RNN to learn a weighted finite automaton (WFA) approximating its behaviour.

\para{Probabilistic Deterministic Finite Automata (PDFAs)} are a weighted variant of DFAs where each state defines a categorical next-token distribution.
Processing a sequence in a PDFA is simple: input tokens are processed one by one, getting the next state and probability for each token by table lookup.

WFAs are non-deterministic and so not immediately analogous to RNNs.
They are also slower to use than PDFAs, as processing each token in an input sequence requires a matrix multiplication.
Finally, spectral learning algorithms are not guaranteed to return stochastic hypotheses even when the target is stochastic --
though this can remedied by using quadratic weighted automata \cite{QWAs} and normalising their weights.
For these reasons we prefer PDFAs over WFAs for RNN approximation.
Formally:

\para{Problem Definition} Given an LM-RNN $R$, find a PDFA $W$ approximating $R$, such that for any prefix $p$ its next-token distributions in $W$ and in $R$ have low total variation distance 
between them.

Existing works on PDFA reconstruction assume a sample based paradigm: the target cannot be queried explicitly for a sequence's probability or conditional probabilities ~\cite{PDFAClarkThollard2004,PDFALearningALERGIA,PDFALearningBalle2013}. As such, these methods cannot take full advantage of the information available from an LM-RNN\footnote{It
is possible to adapt these methods to an active learning setting, 
in which they may query an oracle for exact probabilities.
However, this raises other questions: on which suffixes are prefixes compared? 
How does one pool the probabilities of two prefixes when merging them? 
We leave such an adaptation to future work.}.
Meanwhile, most work on the extraction of finite automata from RNNs has focused on ``binary'' deterministic finite automata (DFAs)~\cite{NNExtraction,NNExtractionFuzzyClustering,NNExtractionGiles2017,OurDFAExtraction,SergioExtraction}, which cannot fully express the behaviour of an LM-RNN.

\para{Our Approach} 
Following the successful application of {\lstar} \cite{Lstar} to RNNs for DFA extraction \cite{OurDFAExtraction}, we develop an adaptation of {\lstar} for the weighted case. 
The adaptation returns a PDFA when applied to a stochastic target such as an LM-RNN. It interacts with an oracle using two types of queries:

\begin{compactenum}
	\item \emph{Membership Queries}: requests to give the target probability of the last token in a sequence.
	\item \emph{Equivalence Queries}: requests to accept or reject a hypothesis PDFA, returning a \emph{counterexample} --- a sequence for which the hypothesis automaton and the target language diverge beyond the tolerance on the next token distribution  --- if rejecting.
\end{compactenum}

The algorithm alternates between filling an \emph{observation table} with observations of the target behaviour, and presenting minimal PDFAs consistent with that table to the oracle for equivalence checking. 
This continues until an automaton is accepted. 
The use of conditional properties 
in the observation table prevents the observations from vanishing to $0$ 
on low probabilities.
To the best of our knowledge, 
this is 
the first work on learning PDFAs from RNNs.

A key insight of our adaptation is the use of an \emph{additive \toltype~tolerance} $t{\in}[0,1]$
when comparing rows in the table. 
In this framework, two probability vectors 
are considered $t$-equal
if their probabilities for each event are within $t$ of each other.
Using this tolerance enables us to 
extract a much smaller PDFA than the original target, while still making locally similar predictions to it on any given 
sequence.
This is necessary because RNN states are real valued vectors, making the potential number of reachable states in an LM-RNN  unbounded.
The tolerance is non-transitive, making construction of PDFAs from the table more challenging than in \lstar. 
Our algorithm suggests a way to address this.

Even with this tolerance, reaching equivalence may take a long time for large target PDFAs, and so we design our algorithm to allow anytime stopping of the extraction. The method allows the extraction to be limited while still maintaining certain guarantees on the reconstructed PDFA.

\emph{Note.} 
While this paper only discusses RNNs, the algorithm itself is actually agnostic to the underlying structure of the target, and can be applied to any language model.
In particular it may be applied to 
transformers \cite{attention2017,BERT2018orig}.
However, in this case the analogy to PDFAs breaks down.

\paragraph{Contributions} The main contributions of this paper are:
\begin{compactenum}
\item An algorithm for reconstructing a WDFA from any given weighted target, and in particular a PDFA if the target is stochastic.
\item A method for anytime extraction termination 
while still maintaining correctness guarantees.
\item An implementation of the algorithm \footnote{Available at \texttt{www.github.com/tech-srl/weighted\textunderscore lstar}} and an evaluation over extraction from LM-RNNs, including a comparison to other LM reconstruction techniques.
\end{compactenum}
\section{Related Work}\label{Se:relatedWork}
In Weiss et al \cite{OurDFAExtraction}, we presented a method for applying Angluin's exact learning algorithm \lstar \cite{Lstar} to RNNs, successfully extracting deterministic finite automata (DFAs) from given binary-classifier RNNs.
This work expands on this by adapting \lstar to extract PDFAs from LM-RNNs. 
To apply exact learning to RNNs, one must implement equivalence queries: requests to accept or reject a hypothesis. Okudono et al. \cite{WFAfromRNNRegression} show how to adapt the equivalence query presented in \cite{OurDFAExtraction} to the weighted case.

There exist many methods for PDFA learning,
originally 
for
acyclic PDFAs ~\cite{PDFALearningAcyclic88,PFALearningAcyclic98,PDFALearningAcyclic99}, and 
later 
for
PDFAs in general ~\cite{PDFAClarkThollard2004, PDFALearningALERGIA,PDFALearningMDI,PDFALearningPACPalmer,PDFALearningPACCastro,PDFALearningBalle2013}.
These methods split and merge states in the learned PDFAs according to sample-based estimations of their conditional distributions.
Unfortunately, they require very large sample sets to succeed (e.g., \cite{PDFAClarkThollard2004} requires \textasciitilde 13m samples for a PDFA with $|Q|,|\Sigma|=2$).

Distributions over $\Sigma^*$ can also be represented by WFAs, though these are non-deterministic. 
These can be learned using \emph{spectral algorithms},
which use SVD decomposition and $|\Sigma|+1$ matrices of observations from the target to build a WFA \cite{SpectralAlgorithmsBailly,SpectralLearningofWFAs,LearningWeightedAutomata,SpectralAlgorithmsHsu}.
Spectral algorithms have recently been applied to RNNs to extract WFAs representing their behaviour ~\cite{WFAfromRNNSpectral,WFAfromRNNRegression,wfas-2rnns}, we compare to \cite{WFAfromRNNSpectral} in this work.
The choice of observations 
used
is also a focus of research in this field \cite{quattonicarrerasspectralbasis}.

For more on language modelling, see the reviews of Goodman~\cite{LMReviewProgress} or Rosenfeld~\cite{LMReviewTwoDecades},
or the Sequence Prediction Challenge (SPiCe)~\cite{SPiCe} and  Probabilistic Automaton Challenge (PAutomaC)~\cite{PAutomaC}.

\section{Background}\label{Se:Background}

\para{Sequences and Notations}
For a finite alphabet $\Sigma$, the set of finite sequences over $\Sigma$ is denoted by $\Sigma^*$, and the empty sequence by $\varepsilon$. 
For any $\Sigma$ and stopping symbol $\stopsym\notin\Sigma$, we denote $\Sigma_\stopsym\triangleq\Sigma\cup\{\stopsym\}$, and  $\Sigma^{+\$}\triangleq\Sigma^*\concat\Sigma_\$$  -- the set of $s\in\Sigma_\$\setminus\{\varepsilon\}$ where the stopping symbol may only appear at the end.

For a sequence $w\in\Sigma^*$, its length is denoted $|w|$, its concatenation after another sequence $u$ is denoted $u\concat w$, its $i$-th element is denoted $w_{i}$, and its prefix of length $k\leq |w|$ is denoted $w_{:k}=w_1\concat ...\concat w_k$.
We use the shorthand $w_{-1}\triangleq w_{|w|}$ 
and $w_{:-1}\triangleq w_{:|w|-1}$.
A set of sequences $S\subseteq\Sigma^*$ is said to be \emph{prefix closed} if for every $w\in S$ and $k\leq |w|$, $w^k\in S$. \emph{Suffix closedness} is defined analogously. 

For any finite alphabet $\Sigma$ and set of sequences $S\subseteq \Sigma^*$, we assume some internal ordering of the set's elements $s_1,s_2,...$ to allow discussion of vectors of observations over those elements.

\para{Probabilistic Deterministic Finite Automata (PDFAs)}
are tuples $A=\langle Q,\Sigma,\delta_Q,q^i,\delta_W\rangle$ such that 
$Q$ is a finite set of states, 
$q^i\in Q$ is the initial state, 
$\Sigma$ is the finite input alphabet, $\delta_Q:Q\times\Sigma\rightarrow Q$ is the transition function and 
$\delta_W:Q\times\Sigma_\$\rightarrow [0,1]$ is the transition weight function, 
satisfying $\sum_{\sigma\in\Sigma_\$}\delta_W(q,\sigma)=1$ for every $q\in Q$.

The recurrent application of $\delta_Q$ to a sequence is denoted by $\hatdelta:Q\times\Sigma^*\rightarrow Q$, and defined: 
$\hatdelta(q,\varepsilon)\triangleq q$ and $\hatdelta(q,w\concat a)\triangleq \delta_Q(\hatdelta(q,w),a)$
for every $q\in Q,a\in\Sigma$, $w\in\Sigma^*$.
We abuse notation to denote: $\hatdelta(w)\triangleq\hatdelta(q^i,w)$ for every $w\in\Sigma^*$. 
If for every $q\in Q$ there exists a series of non-zero transitions reaching a state $q$ with $\delta_W(q,\$)>0$, then $A$ defines a distribution $P_A$ over $\Sigma^*$ as follows: 
for every $w\in\Sigma^*$, $P_A(w)=\delta_W(\hatdelta(w),\$)\cdot\prod_{i\leq |w|} \delta_W(\hatdelta (w_{:i-1}),w_i)$.

\para{Language Models (LMs)} Given a finite alphabet $\Sigma$, a \emph{language model} $M$ over $\Sigma$ is a model defining a distribution $P_M$ over $\Sigma^*$.  
For any $w\in\Sigma^*,S\subset\Sigma^{+\$}$, and $\sigma\in\Sigma$, $P=P_M$ induces the following:
\begin{compactitem}
\item \emph{Prefix Probability:} $P^p(w)\triangleq\sum_{v\in\Sigma^*}P(w\concat v)$. 
\item \emph{Last Token Probability:}
if $P^p(w)>0$, then $P^l(w\concat\sigma)\triangleq\frac{P^p(w\concat\sigma)}{P^p(w)}$ and $P^l(w\concat\$)\triangleq\frac{P(w)}{P^p(w)}$.
\item \emph{Last Token Probabilities Vector:}
if $P^p(w)>0$, $P^l_S(w)\triangleq(P^l(w\concat s_1),...,P^l(w\concat s_{|S|}))$.
\item \emph{Next Token Distribution:} $P^n(w):\Sigma_\$\rightarrow [0,1]$, defined: $P^n(w)(\sigma)=P^l(w\concat\sigma)$.
\end{compactitem}

\para{Variation Tolerance} Given two categorical distributions $\mathbf{p}$ and $\mathbf{q}$, their total variation distance is defined $\delta(\mathbf{p},\mathbf{q}) \triangleq \|\mathbf{p}-\mathbf{q}\|_\infty$, i.e., the largest difference in probabilities that they assign to the same event. 
Our algorithm tolerates some variation distance between next-token probabilities, as follows:

Two event probabilities $p_1,p_2$ are called \emph{$t$-equal} and denoted $p_1\tapprox p_2$ if $|p_1-p_2|\leq t$.
Similarly, two vectors of probabilities 
$\mathbf{v_1},\mathbf{v_2}\in [0,1]^n$ are called \emph{$t$-equal} and denoted $\mathbf{v_1}\tapprox \mathbf{v_2}$ if $\|\mathbf{v_1}-\mathbf{v_2}\|_\infty\leq t$, i.e. if $\underset{i\in[n]}{\max} (|\mathbf{v_{1_i}}-\mathbf{v_{2_i}}|)\leq t$.
For any distribution $P$ over $\Sigma^*$, $S\subset \Sigma^{+\$}$, and $p_1,p_2\in\Sigma^*$, we denote $p_1\approxbar{(P,S,t)} p_2$ if $P^l_{S}(p_1)\tapprox P^l_{S}(p_2)$, or simply $p_1\approxbar{(S,t)} p_2$ if $P$ is clear from context.
For any two language models $A,B$ over $\Sigma^*$ and $w\in\Sigma^{+\$}$, we say that $A,B$ are \emph{$t$-consistent on $w$} if $P_A^l(u)\tapprox P_B^l(u)$ for every prefix $u\neq\varepsilon$ of $w$. We call $t$ the \emph{\toltype tolerance}.

\para{Oracles and Observation Tables}
Given an oracle $\oracle$,
an observation table for $\oracle$ is a sequence indexed matrix $O_{P,S}$ of observations taken from it, 
with the rows indexed by prefixes $P$ and the columns by suffixes $S$.  
The observations are $O_{P,S}(p,s)=\oracle(p\concat s)$ for every $p\in P$, $s\in S$.
For any $p\in\Sigma^*$ we denote $\oracle_S(p)\triangleq (\oracle(p\concat s_1),...,\oracle(p\concat s_2))$,
and for every $p\in P$ the $p$-th row in $O_{P,S}$ is denoted $O_{P,S}(p)\triangleq \oracle_S(p)$. 
In this work we use an oracle for the last-token probabilities of the target, $\oracle (w)=P^l(w)$ for every $w\in\Sigma^{+\$}$, and maintain $S\subseteq \Sigma^{+\$}$.

\para{Recurrent Neural Networks (RNNs)}
An RNN is a recursive parametrised function $h_t = f(x_t,h_{t-1})$ with initial state $h_0$, such that $h_t\in \mathbb{R}^n$ is the state after time $t$ and $x_t\in X$ is the input at time $t$.
A language model RNN (LM-RNN) over an alphabet $X=\Sigma$ is an RNN coupled with a prediction function $g:h\mapsto d$, where $d\in [0,1]^{|\Sigma_\$|}$ is a vector representation of a next-token distribution.
RNNs differ from PDFAs only in that their number of reachable states (and so number of different next-token distributions for sequences) may be unbounded.

\section{Learning PDFAs with Queries and Counterexamples} \label{OurAlg}
In this section we describe the details of our algorithm.
We explain why a direct application of {\lstar} to PDFAs will not work, and then present our non-trivial adaptation.
Our adaptation does not rely on the target being stochastic, and can in fact be applied to reconstruct any WDFA from an oracle.

\para{Direct application of {\lstar} does not work for LM-RNNs:} {\lstar} is a polynomial-time algorithm for learning a deterministic finite automaton (DFA) from an oracle.
It can be adapted to work with oracles giving any finite number of classifications to sequences, and can be naively adapted to a probabilistic target $P$ with finite possible next-token distributions $\{P^n(w)|w\in\Sigma^*\}$ by treating each next-token distribution as a sequence classification.
However, \emph{this will not work for reconstruction from RNNs}. This is because the set of reachable states in a given RNN is unbounded, and so also the set of next-token distributions. Thus, in order to practically adapt {\lstar} to extract PDFAs from LM-RNNs, we must reduce the number of classes {\lstar} deals with. 

\para{Variation Tolerance}
Our algorithm reduces the number of classes it considers by allowing an additive {\toltype} tolerance $t\in [0,1]$, and considering $t$-equality (as presented in Section \ref{Se:Background}) as opposed to actual equality when comparing probabilities. 
In introducing this tolerance we must handle the fact that it may be non-transitive: 
there may exist $a,b,c\in [0,1]$ such that $a\tapprox b, b\tapprox c$, but $a \ntapprox c$.
\footnote{We could define a {\toltype} tolerance by quantisation of the distribution space, which would be transitive. However this may be unnecessarily aggressive at the edges of the intervals.}

To avoid potentially grouping together all 
predictions on 
long sequences, which are likely to have very low probabilities, our algorithm observes only local probabilities.
In particular, the algorithm uses an oracle that gives the last-token probability for every non-empty input sequence. 

\subsection{The Algorithm}\label{OurAlg:cexdef}

The algorithm loops over three main steps:
\begin{inparaenum}[(1)]
\item expanding an observation table $O_{P,S}$ until it is closed and consistent, 
\item constructing a hypothesis automaton, and 
\item making an equivalence query about the hypothesis. 
\end{inparaenum}
The loop repeats as long as the oracle returns counterexamples for the hypotheses. 
In our setting, 
counterexamples are sequences $w\in\Sigma^*$ after which the hypothesis and the target have next-token distributions that are not $t$-equal.
They are handled by adding all of their prefixes to $P$. 

Our algorithm expects last token probabilities from the oracle, i.e.: $\oracle(w)=P^l_T(w)$ where $P_T$ is the target distribution. 
The oracle is not queried on $P^l_T(\varepsilon)$, which is undefined. 
To observe the entirety of every prefix's next-token distribution, $O_{P,S}$ is initiated with $P=\{\varepsilon\},S=\Sigma_\$$.

\paragraph{Step 1: Expanding the observation table}\label{OurAlg:obs}
$O_{P,S}$ is expanded as in \lstar~\cite{Lstar}, but with the definition of row equality relaxed. Precisely, it is expanded until:
\begin{compactenum}
\item \emph{Closedness} For every $p_1\in P$ and $\sigma\in\Sigma$, there exists some $p_2\in P$ such that $p_1\concat\sigma\approxbar{S,t}p_2$.
\item \emph{Consistency} For every $p_1,p_2\in P$ such that $p_1\approxbar{S,t}p_2$, for every $\sigma\in\Sigma$, $p_1\concat\sigma\approxbar{S,t}p_2\concat\sigma$.
\end{compactenum}

The table expansion is managed by a queue $L$ initiated to $P$, from which prefixes $p$ are processed one at a time as follows:
If $p\notin P$, and there is no $p'\in P$ s.t. $p\approxbar{(t,S)}p'$, then $p$ is added to $P$. If $p\in P$ already, then it is checked for inconsistency,
i.e. whether there exist $p',\sigma$ s.t. $p\approxbar{(t,S)}p'$ but $p\concat\sigma\napproxbar{(t,S)}p'\concat\sigma$. 
In this case
a \emph{separating suffix} $\tilde{s}$, $P^l_T(p\concat\sigma\concat\tilde{s})\ntapprox P^l_T(p'\concat\sigma\concat\tilde{s})$
is added to $S$,
such that now $p\napproxbar{t,S}p'$, and the expansion restarts. 
Finally, if $p\in P$ then $L$ is updated with $p\concat\Sigma$.

As in \lstar, checking closedness and consistency can be done in arbitrary order.
However, if the algorithm may be terminated before $O_{P,S}$ is closed and consistent, it is better to process
$L$
in order of prefix probability (see section \ref{Stoppe}).

\paragraph{Step 2: PDFA construction}\label{OurAlg:cons}
Intuitively, we would like to group equivalent rows of the observation table to form the states of the PDFA, and map transitions between these groups according to the table's observations. The challenge in the \toltypeng-tolerating setting is that \emph{$t$-equality is not transitive}. 

Formally, let $C$ be a partitioning (\emph{clustering}) of $P$, and for each $p\in P$ let $c(p)\in C$ be the partition (\emph{cluster}) containing $p$. $C$ should satisfy:
\begin{compactenum}
\item \emph{Determinism} 
For every $c\in C$, $p_1,p_2\in c$, $\sigma\in\Sigma$: $p_1\concat\sigma,p_2\concat\sigma\in P \implies c(p_1\concat\sigma)=c(p_2\concat\sigma)$.
\item \emph{$t$-equality (Cliques)} For every $c\in C$ and $p_1,p_2\in c$, $p_1  \approxbar{(t,S)} p_2$.
\end{compactenum}

For $c\in C$, $\sigma \in \Sigma$, we denote $C_{c,\sigma}=\{c(p\concat \sigma) | p\in c, p\concat\sigma\in P \}$ the next-clusters reached from $c$ with $\sigma$, and $k_{c,\sigma}\triangleq|C_{c,\sigma} |$. Note that $C$ satisfies determinism iff $k_{c,\sigma}\leq1$ for every $c\in C,\sigma\in\Sigma$. Note also that the constraints are always satisfiable by the clustering $C=\{\{p\}\}_{p\in P}$

We present a 4-step algorithm to solve these constraints while trying to avoid excessive partitions:
\footnote{We describe our implementation of these stages in appendix \ref{App:Implementation}.}

\begin{compactenum}
\item \emph{Initialisation}: The prefixes $p\in P$ are partitioned into some initial clustering $C$ according to the $t$-equality of their rows, $O_S(p)$. 
\item \emph{Determinism I}: 
$C$ is refined until it satisfies determinism:
clusters $c\in C$ with tokens $\sigma$ for which $k_{c,\sigma}>1$ are split by next-cluster equivalence into $k_{c,\sigma}$ new clusters. 
\item \emph{Cliques}: Each cluster is refined into cliques (with respect to $t$-equality). 
\item \emph{Determinism II}: $C$ is again refined until it satisfies determinism, as in (2).
\end{compactenum}

Note that refining a partitioning into cliques may break determinism, but refining into a deterministic partitioning will not break cliques. In addition, when only allowed to refine clusters (and not merge them), all determinism refinements are necessary. Hence the order of the last 3 stages. 

Once the clustering $C$ is found, a PDFA $\mathcal{A}=\langle C,\Sigma,\delta_Q,c(\varepsilon),\delta_W \rangle$ is constructed from it.
Where possible, $\delta_Q$ is defined directly by $C$: for every $p\concat\sigma\in P$, $\delta_Q(c(p),\sigma)\triangleq c(p\concat\sigma)$. 
For $c,\sigma$ for which $k_{c,\sigma}=0$, $\delta_Q(c,\sigma)$ is set as the best cluster match for $p\concat\sigma$, where $p=\argmax_{p\in c} P_T^p(p)$. This is chosen according to the heuristics presented in Section ~\ref{force-heuristics}.
The weights $\delta_W$ are defined as follows: for every $c\in C, \sigma\in\Sigma_\$$, 
$\delta_W(c,\sigma)\triangleq 
\frac
{\sum_{p\in c} P_T^p(p)\cdot P_T^l(p\concat\sigma)}
{\sum_{p\in c} P_T^p(p)}$.

\paragraph{Step 3: Answering Equivalence Queries} 
We sample the target LM-RNN and hypothesis PDFA $\mathcal{A}$ a finite
number of times, testing every prefix of each sample to see if it is a counterexample.
If none is found, we accept $\mathcal{A}$.
Though simple, we find this method to be sufficiently effective in practice.
A more sophisticated approach is 
presented in 
\cite{WFAfromRNNRegression}.

\subsection{Practical Considerations} \label{Practical}

We present some methods and heuristics that allow a more effective application of the algorithm to large (with respect to $|\Sigma|$, $|Q|$) or poorly learned grammars. 

\para{Anytime Stopping} \label{Stoppe}
In case the algorithm runs for too long, 
we allow termination before $O_{P,S}$ is closed and consistent, 
which may be imposed by size or time limits on the table expansion. 
If $|S|$ reaches its limit, the table expansion continues but stops checking consistency.
If the time or $|P|$ limits are reached, 
the algorithm stops, constructing and accepting a PDFA from the table as is. 
The construction is unchanged up to the fact that some of the transitions may not have a defined destination, for these we use a ``best cluster match'' as described in section \ref{force-heuristics}.
This does not harm the guarantees on $t$-consistency between $O_{P,S}$ and the returned PDFA discussed in Section \ref{Se:Guarantees}.

\para{Order of Expansion} As some prefixes will not be added to $P$ under anytime stopping, the order in which rows are checked for closedness and consistency matters. 
We sort $L$ by prefix weight.
Moreover, if a prefix $p_1$ being considered is found inconsistent w.r.t. some $p_2\in P,\sigma\in\Sigma_\$$, then all such pairs $p_2,\sigma$ are considered and the separating suffix $\tilde{s}\in\sigma\concat S$, $\oracle(p_1\concat\tilde{s})\napproxbar{t}\oracle(p_2\concat\tilde{s})$ with the highest minimum conditional probability ${\max}_{p_2}{\min}_{i=1,2}\frac{P^p_T(p_i\concat\tilde{s})}{P^p_T(p_i})$  is added to $S$.

\paragraph{Best Cluster Match}\label{force-heuristics}
Given a prefix $p\notin P$ and set of clusters $C$, we seek a best fit $c\in C$ for $p$. First we filter $C$ for the following qualities until one is non-empty, in order of preference:
\begin{inparaenum}[(1)]
\item $c'=c\cup\{p\}$ is a clique w.r.t. $t$-equality.
\item There exists some $p'\in c$ such that $p'\approxbar{(t,S)} p$, and $c$ is not a clique.
\item There exists some $p'\in c$ such that $p' \approxbar{(t,S)} p$.
\end{inparaenum}
If no clusters satisfy these qualities, we remain with $C$. 
From the resulting group $C'$ of potential matches, the best match could be the cluster $c$ minimising $||O_S(p')-O_S(p)||_\infty$, $p'\in c$. 
In practice, we choose from $C'$ arbitrarily for efficiency.

\paragraph{Suffix and Prefix Thresholds}\label{suffprefthresh}
Occasionally when checking the consistency of two rows $p_1\tapprox p_2$, a separating suffix $\sigma\concat s\in \Sigma\concat S$ will be found that is actually very unlikely to be seen after $p_1$ or $p_2$. 
In this case it is unproductive to add $\sigma\concat s$ to $S$. 
Moreover -- especially as RNNs are unlikely to perfectly learn a probability of $0$ for some event -- it is possible that going through $\sigma\concat s$ will reach a large number of `junk' states. 
Similarly when considering a prefix $p$, if $P^l_T(p)$ is very low then it is possible that it is the failed encoding of probability $0$, and that all states reachable through $p$ are not useful.

We introduce thresholds $\varepsilon_S$ and $\varepsilon_P$ for both suffixes and prefixes. When a potential separating suffix $\tilde{s}$ is found from prefixes $p_1$ and $p_2$, it is added to $S$ only if
${\min}_{i=1,2}\nicefrac{P^p(p_i\concat \tilde{s})}{P^p(p_i)}\geq \varepsilon_S$. 
Similarly, potential new rows $p\notin P$ are only added to $P$ if $P^l(p)\geq \varepsilon_P$.

\paragraph{Finding Close Rows}\label{kdtree}
We maintain $P$ in a KD-tree $T$ indexed by row entries $O_{P,S}(p)$, with one level for every column $s\in S$.  
When considering
of a prefix $p\concat\sigma$, we use $T$ to get the subset of all potentially $t$-equal prefixes. 
$T$'s levels are split into equal-length intervals, we find $2t$ to work well.

\paragraph{Choosing the Variation Tolerance}\label{setting_tol}
In our initial experiments (on SPiCes  0-3), we used $t=\nicefrac{1}{|\Sigma|}$. The intuition was that given no data, the fairest distribution over $|\Sigma|$ is the uniform distribution, and so this may also be a reasonable threshold for a significant difference between two probabilities. In practice, we found that $t=0.1$ often strongly differentiates states even in models with larger alphabets -- except for SPiCe 1, where $t=0.1$ quickly accepted a model of size 1. A reasonable strategy for choosing $t$ is to begin with a large one, and reduce it if equivalence is reached too quickly. 
\section{Guarantees}\label{Se:Guarantees}
We note some guarantees on the extracted model's qualities and relation to its target model. 
\emph{Formal statements and full proofs for each of the guarantees listed here are given in appendix \ref{App:Guarantees}.}

\paragraph{Model Qualities}
The model is guaranteed to be deterministic by construction. Moreover, if the target is stochastic, then the returned model is guaranteed to be stochastic as well.

\paragraph{Reaching Equivalence}
If the algorithm terminates successfully (i.e., having passed an equivalence query), then the returned model is $t$-consistent with the target on every sequence $w\in\Sigma^*$, 
by definition of the query.
In practice we have no true oracle and only approximate equivalence queries by sampling the models, and so can only attain a probable guarantee of their relative $t$-consistency.

\paragraph{$t$-Consistency and Progress}
No matter when the algorithm is stopped,
the returned model is always $t$-consistent with its target on every $p\in P\concat\Sigma_\$$, where $P$ is the set of prefixes in the table $O_{P,S}$. 
Moreover, as long as the algorithm is running, the prefix set $P$
is always increased within a finite number of operations.
This means that the algorithm maintains a growing set of prefixes on which any PDFA it returns is guaranteed to be $t$-consistent with the target. 
In particular, this means that if
equivalence is not reached, at least
\emph{the algorithm's model of the target improves for as long~as~it~runs}.
\section{Experimental Evaluation}\label{Se:Results}

We apply our algorithm to 2-layer LSTMs trained on grammars from the SPiCe competition~\cite{SPiCe},
adaptations of the Tomita grammars~\cite{tomita82} to PDFAs,
and small PDFAs representing languages with unbounded history.
The LSTMs have input dimensions $2$-$60$ and hidden dimensions $20$-$100$. The LSTMs and their  training methods are fully described in Appendix ~\ref{App:RNNs}.

\para{Compared Methods}
We compare our algorithm to the sample-based method ALERGIA~\cite{PDFALearningALERGIA}, 
the spectral algorithm used in ~\cite{WFAfromRNNSpectral},
and $n$-grams. 
An $n$-gram is a PDFA whose states are a sliding window of length $n-1$ over the input sequence,  
with transition function $\sigma_1\concat...\concat\sigma_n,\sigma\mapsto\sigma_2\concat...\sigma_n\concat\sigma$. 
The probability of a token $\sigma$ from state $s\in\Sigma^{n-1}$ is the MLE estimate $\frac{N(s\concat\sigma)}{N(s)}$, where $N(w)$ is the number of times the sequence $w$ appears as a subsequence in the samples.
For ALERGIA, we use the PDFA/DFA inference toolkit \flexfringe ~\cite{flexfringe}.

\para{Target Languages}
We train $10$ RNNs on a subset of the SPiCe grammars, covering languages generated by HMMs, and languages from the NLP, \emph{software}, and \emph{biology} domains. 
We train $7$ RNNs on PDFA adaptations of the $7$ Tomita languages~\cite{tomita82},
made from the minimal DFA for each language by giving each of its states a next-token distribution as a function of whether it is accepting or not.
We give a full description of the Tomita adaptations and extraction results in appendix \ref{App:smalls}.
As we show in (\ref{subse:results}), the $n$-gram models prove to be very strong competitors on the {\spice} languages. To this end, we consider three additional languages that need to track information for an unbounded history, and thus cannot be captured by \emph{any} $n$-gram model. We call these UHLs (unbounded history languages).

\uhl s 1 and 2 are PDFAs that cycle through 9 and 5 states with different next token probabilities.
{\uhl} 3 is a weighted adaptation of the 5\textsuperscript{th} Tomita grammar, changing its next-token distribution according to the parity of the seen \texttt{0}s and \texttt{1}s.
The \uhl s are drawn in appendix \ref{App:smalls}.

\para{Extraction Parameters}
Most of the extraction parameters differ between the RNNs, and are described in the results tables (\ref{tab:spice}, \ref{tab:uhl}).
For our algorithm, we always limited the equivalence query to $500$ samples.
For the spectral algorithm, 
we made WFAs for all ranks $k\in[50], k=50m, m\in [10]$, $k=100m, m\in[10]$, and $k=rank(H)$.
For the $n$-grams we used all $n\in[6]$.
For these two, we always show the best results for NDCG and WER.
For ALERGIA in the \flexfringe toolkit, we use the parameters
\texttt{symbol\textunderscore count=50} and \texttt{state\textunderscore count=N}, with \texttt{N} given in the tables.

\para{Evaluation Measures} We evaluate the extracted models against their target RNNs on word error rate (WER) and on normalised discounted cumulative gain (NDCG), which was the scoring function for the SPiCe challenge. In particular the SPiCe challenge evaluated models on $NDCG_5$, and we evaluate the models extracted from the SPiCe RNNs on this as well. 
For the \uhl s, we use $NDCG_2$ as they have smaller alphabets.
We do not use probabilistic measures such as perplexity, as the spectral algorithm is not guaranteed to return probabilistic automata. 

\begin{compactenum}
\item \emph{Word error rate (WER)}: The WER of model A against B on a set of predictions is the fraction of next-token predictions (most likely next token) that are different in A and B.
\item \emph{Normalised discounted cumulative gain (NDCG)}: 
The NDCG of A against B on a set of sequences $\{w\}$ scores A's ranking of the top $k$ most likely tokens after each sequence $w$, $a_1,...,a_k$, in comparison to the actual most likely tokens given by B, $b_1,...,b_k$.
Formally: $$NDCG_k(a_1,...,a_k)= \nicefrac{\sum_{n\in [k]}\frac{P^l_B(w\concat a_n)}{\log_2(n+1)}}{\sum_{n\in [k]}\frac{P^l_B(w\concat b_n)}{\log_2(n+1)}}$$ 

\end{compactenum}
For NDCG we sample the RNN repeatedly, taking all the prefixes of each sample until we have $2000$
prefixes.
We then compute the NDCG for each prefix and  take the average. For WER, we take $2000$ full samples from the RNN, and return the fraction of errors over all of the next-token predictions in those samples.
An ideal WER and NDCG is $0$ and $1$, we note this with $\downarrow,\uparrow$ in the tables.

\subsection{Results and Discussion}\label{subse:results}

Tables \ref{tab:spice} and \ref{tab:uhl} show the results of extraction from the SPiCe and {\uhl} RNNs, respectively. In them, we list our algorithm as W\lstar (Weighted \lstar).
For the WFAs and $n$-grams, which are generated with several values of $k$ (rank) and $n$, we show the best scores for each metric. We list the size of the best model for each metric.
We do not report the extraction times separately, as they are very similar: the majority of time in these algorithms is spent generating the samples or Hankel matrices.

For PDFAs and WFAs the size columns present the number of states, for the WFAs this is equal to the rank $k$ with which they were reconstructed. For $n$-grams the size is the number of table entries in the model, and the chosen value of $n$ is listed in brackets. 
In the {\spice} languages, our algorithm did not reach equivalence, and used between 1 and 6 counterexamples for every language before being stopped --
with the exception of \spice 1 with $t=0.1$, which reached equivalence on a single state.
The \uhl s and \tomita s used 0-2 counterexamples each before reaching equivalence. 

The SPiCe results show a strong advantage to our algorithm in most of the small synthetic languages (1-3), with the spectral extraction taking a slight lead on SPiCe 0. 
However, in the remaining SPiCe languages, the $n$-gram strongly outperforms all other methods. 
Nevertheless, $n$-gram models are inherently restricted to languages that can be captured with
bounded histories, and the \uhl s demonstrate cases where this property does not hold. Indeed, all the algorithms outperform the $n$-grams
on these languages (Table \ref{tab:uhl}).

\begin{table}
	\centering
\small
\begin{tabular}{c l|| r r r c c}
	Language ($|\Sigma|,\ell$)& Model & WER
	$\downarrow$ & NDCG$\uparrow$ & Time (h) & WER Size & NDCG Size \\
	\hline

\hline
{\spice} 0 ($ 4$, $ 1.15 $) & W\lstar & 0.084 & 0.987 & 0.3 &  4988 \cexs{4} & 4988 \cexs{4}  \\
 & Spectral & \textbf{0.053} & \textbf{0.996} & 0.3 & k=150 & k=200 \\
 & N-Gram & 0.096 & 0.991 & 0.8 &  1118 (n=6) & 1118 (n=6)  \\
 & ALERGIA\ignore{ nsamples/heuristic/statecount, wer: (5000000, 'alergia', 5000) ndcg: (5000000, 'alergia', 5000) }  & 0.353 & 0.961 & 2.9 &  66 & 66  \\
\hline
{\spice} 1 ($ 20$, $ 2.77 $) & W\lstar{\boldmath$\dagger$} & \textbf{0.093} & \textbf{0.971} & 0.4 &  152 \cexs{6} & 152 \cexs{6}  \\
 & W\lstar & 0.376 & 0.891 & 0.1 &  1 \cexs{0} & 1 \cexs{0}  \\
 & Spectral & 0.319 & 0.909 & 2.9 & k=12 & k=11 \\
 & N-Gram & 0.337 & 0.897 & 0.8 & 8421 (n=4) & 421 (n=3) \\
 & ALERGIA\ignore{ nsamples/heuristic/statecount, wer: (5000000, 'alergia', 5000) ndcg: (5000000, 'alergia', 5000) }  & 0.376 & 0.892 & 1.2 &  7 & 7  \\
\hline
{\spice} 2 ($ 10$, $ 2.13 $) & W\lstar{\boldmath$\ddagger$} & \textbf{0.08} & \textbf{0.972} & 0.8 &  962 \cexs{7} & 962 \cexs{7}  \\
 & Spectral & 0.263 & 0.893 & 1.6 & k=7 & k=5 \\
 & N-Gram & 0.278 & 0.894 & 0.8 &  1111 (n=4) & 1111 (n=4)  \\
 & ALERGIA\ignore{ nsamples/heuristic/statecount, wer: (5000000, 'alergia', 5000) ndcg: (5000000, 'alergia', 5000) }  & 0.419 & 0.844 & 1.2 &  11 & 11  \\
\hline
{\spice} 3 ($ 10$, $ 2.15 $) & W\lstar{\boldmath$\ddagger$} & \textbf{0.327} & \textbf{0.928} & 1.0 &  675 \cexs{6} & 675 \cexs{6}  \\
 & Spectral & 0.466 & 0.843 & 1.2 & k=6 & k=8 \\
 & N-Gram & 0.46 & 0.847 & 0.8 & 1111 (n=4) & 11110 (n=5) \\
 & ALERGIA\ignore{ nsamples/heuristic/statecount, wer: (5000000, 'alergia', 10000) ndcg: (5000000, 'alergia', 10000) } 	{\boldmath$\ddagger\ddagger$} & 0.679 & 0.79 & 1.2 &  8 & 8  \\
\hline
{\spice} 4 ($ 33$, $ 1.73 $) & W\lstar & 0.301 & 0.829 & 0.7 &  4999 \cexs{1} & 4999 \cexs{1}  \\
 & Spectral & 0.453 & 0.727 & 1.2 & k=450 & k=250 \\
 & N-Gram & \textbf{0.099} & \textbf{0.968} & 0.8 & 186601 (n=6) & 61851 (n=5) \\
 & ALERGIA\ignore{ nsamples/heuristic/statecount, wer: (5000000, 'alergia', 10000) ndcg: (5000000, 'alergia', 10000) } {\boldmath$\ddagger\ddagger$} & 0.639 & 0.646 & 4.4 &  42 & 42  \\
\hline
{\spice} 6 ($ 60$, $ 1.66 $) & W\lstar & 0.593 & 0.644 & 2.5 &  5000 \cexs{1}  & 5000 \cexs{1}   \\
 & Spectral & 0.705 & 0.535 & 6.1 & k=17 & k=32 \\
 & N-Gram & \textbf{0.285} & \textbf{0.888} & 0.8 &  127817 (n=5) & 127817 (n=5)  \\
 & ALERGIA\ignore{ nsamples/heuristic/statecount, wer: (5000000, 'alergia', 5000) ndcg: (5000000, 'alergia', 5000) }  & 0.687 & 0.538 & 1.9 &  26 & 26  \\
\hline
{\spice} 7 ($ 20$, $ 1.8 $) & W\lstar & 0.626 & 0.642 & 0.5 &  4996 \cexs{3} & 4996 \cexs{3}  \\
 & Spectral & 0.801 & 0.472 & 2.4 & k=50 & k=27 \\
 & N-Gram & \textbf{0.441} & \textbf{0.812} & 0.7 &  133026 (n=5) & 133026 (n=5)  \\
 & ALERGIA\ignore{ nsamples/heuristic/statecount, wer: (5000000, 'alergia', 5000) ndcg: (5000000, 'alergia', 5000) }  & 0.735 & 0.569 & 1.4 &  8 & 8  \\
\hline
{\spice} 9 ($ 11$, $ 1.15 $) & W\lstar & 0.503 & 0.721 & 0.5 &  4992 \cexs{3} & 4992 \cexs{3}  \\
 & Spectral & 0.303 & 0.877 & 1.9 &  k=44 & k=44  \\
 & N-Gram & \textbf{0.123} & \textbf{0.961} & 1.0 &  44533 (n=6) & 44533 (n=6)  \\
 & ALERGIA\ignore{ nsamples/heuristic/statecount, wer: (5000000, 'alergia', 5000) ndcg: (5000000, 'alergia', 5000) }  & 0.501 & 0.739 & 1.1 &  44 & 44  \\
\hline
{\spice} 10 ($ 20$, $ 2.1 $) & W\lstar & 0.651 & 0.593 & 0.9 &  4987 \cexs{6} & 4987 \cexs{6}  \\
 & Spectral & 0.845 & 0.4 & 1.7 & k=42 & k=41 \\
 & N-Gram & \textbf{0.348} & \textbf{0.845} & 0.8 &  153688 (n=5) & 153688 (n=5)  \\
 & ALERGIA\ignore{ nsamples/heuristic/statecount, wer: (5000000, 'alergia', 5000) ndcg: (5000000, 'alergia', 5000) }  & 0.81 & 0.51 & 2.0 &  13 & 13  \\
\hline
{\spice} 14 ($ 27$, $ 0.89 $) & W\lstar & 0.442 & 0.716 & 0.8 &  4999 \cexs{2} & 4999 \cexs{2}  \\
 & Spectral{\boldmath$\dagger\dagger$} & 0.531 & 0.653 & 2.4 &  k=100 & k=100  \\
 & N-Gram & \textbf{0.079} & \textbf{0.977} & 0.7 & 125572 (n=6) & 46158 (n=5) \\
 & ALERGIA\ignore{ nsamples/heuristic/statecount, wer: (5000000, 'alergia', 10000) ndcg: (5000000, 'alergia', 10000) } {\boldmath$\ddagger\ddagger$} & 0.641 & 0.611 & 1.2 &  19 & 19  \\
\end{tabular}
	\caption{\spice ~results. Each language is listed with its alphabet size $|\Sigma|$ and RNN test loss $\ell$. The $n$-grams and sample-based PDFAs were created from 5,000,000  samples, 			and shared samples. \flexfringe was run with state\textunderscore count${=}5000$. 			 Our algorithm was run with $t{=}0.1,\varepsilon_P,\varepsilon_S{=}0.01,|P|{\leq}5000$ and $|S|{\leq}100$,		 and spectral with $|P|,|S|{=} 1000 $, with some exceptions: {\boldmath$\dagger$}:$t{=}0.05, \varepsilon_S,\varepsilon_P{=}0.0$, {\boldmath$\ddagger$}:$\varepsilon_S{=}0$,
	{\boldmath$\dagger\dagger$}:$|P|,|S|{=}750$, {\boldmath$\ddagger\ddagger$}:state\textunderscore count${=}10,000$.}
	\label{tab:spice}

\end{table}

Our algorithm succeeds in perfectly reconstructing the target PDFA structure for each of the {\uhl} languages, and giving it transition weights within the given \toltype tolerance 
(when extracting from the RNN and not directly from the original target, the weights can only be as good as the RNN has learned).
The sample-based PDFA learning method, ALERGIA, achieved good WER and NDCG scores but did not manage to reconstruct the original PDFA structure. This may be improved by taking a larger sample size, though it comes at the cost of efficiency.

\begin{table}
	\centering
\small
\begin{tabular}{c l|| r r r c c}
	Language ($|\Sigma|,\ell$)& Model & WER$\downarrow$ & NDCG$\uparrow$ & Time (s) & WER Size & NDCG Size \\
	\hline

\hline
{\uhl} 1 \ignore{2, (3,)} ($ 2$, $ 0.72 $) & W\lstar & \textbf{0.0} & \textbf{1.0} & 15 &  9 \cexs{2} & 9 \cexs{2}  \\
 & Spectral & \textbf{0.0} & \textbf{1.0} & 56 & k=80 & k=150 \\
 & N-Gram & 0.129 & 0.966 & 259 &  63 (n=6) & 63 (n=6)  \\
 & ALERGIA\ignore{ nsamples/heuristic/statecount, wer: (500000, 'alergia', 50) ndcg: (500000, 'alergia', 50) }  & 0.004 & 0.999 & 278 &  56 & 56  \\
\hline
{\uhl} 2 \ignore{3, (5,)} ($ 5$, $ 1.32 $) & W\lstar & \textbf{0.0} & \textbf{1.0} & 73 &  5 \cexs{0} & 5 \cexs{0}  \\
 & Spectral & 0.002 & 1.0 & 126 & k=49 & k=47 \\
 & N-Gram & 0.12 & 0.94 & 269 &  3859 (n=6) & 3859 (n=6)  \\
 & ALERGIA\ignore{ nsamples/heuristic/statecount, wer: (500000, 'alergia', 50) ndcg: (500000, 'alergia', 50) }  & 0.023 & 0.979 & 329 &  25 & 25  \\
\hline
{\uhl} 3 \ignore{7, (0.525, 0.425)} ($ 2$, $ 0.86 $) & W\lstar & \textbf{0.0} & \textbf{1.0} & 55 &  4 \cexs{1} & 4 \cexs{1}  \\
 & Spectral & \textbf{0.0} & \textbf{1.0} & 71 & k=44 & k=17 \\
 & N-Gram & 0.189 & 0.991 & 268 &  63 (n=6) & 63 (n=6)  \\
 & ALERGIA\ignore{ nsamples/heuristic/statecount, wer: (500000, 'alergia', 50) ndcg: (500000, 'alergia', 50) }  & 0.02 & 0.999 & 319 &  47 & 47  \\
\end{tabular}
	\caption{\uhl ~results. Each language is listed with its alphabet size $|\Sigma|$ and RNN test loss $\ell$. The $n$-grams and sample-based PDFAs were created from 500,000  samples, 			and shared samples. \flexfringe was run with state\textunderscore count =  50 . 			 Our algorithm was run with $t{=}0.1,\varepsilon_P,\varepsilon_S{=}0.01,|P|{\leq}5000$ and $|S|{\leq}100$,		 and spectral with $|P|,|S|{=} 250 $.}
	\label{tab:uhl}
\end{table}

\para{Tomita Grammars} The full results for the Tomita extractions are given in Appendix \ref{App:smalls}. 

All of the methods  reconstruct them with perfect or near-perfect WER and NDCG, except for $n$-gram which sometimes fails.
For each of the Tomita RNNs, our algorithm extracted and accepted a PDFA with identical structure to the original target in approximately 1 minute (the majority of this time was spent on sampling the RNN and hypothesis before accepting the equivalence query). These PDFAs had transition weights within the variation tolerance of the corresponding target transition weights.

\paragraph{On the effectiveness of n-grams} The n-gram models prove to be a very strong competitors for many of the languages.
Indeed, n-gram models are very effective for learning in cases where the underlying languages have strong local properties, or can be well approximated using local properties, which is rather common (see e.g., Sharan et al. \cite{sharan16}). However, there are many languages, including ones that can be modeled with PDFAs, for which the locality property does not hold, as demonstrated by the {\uhl} experiments.

As $n$-grams are merely tables of observed samples, they are very quick to create. However, their simplicity also works against them: the table grows exponentially in $n$ and polynomially in $|\Sigma|$.
In the future, we hope that our algorithm can serve as a base for creating reasonably sized finite state machines that will be competitive on real world tasks.

\section{Conclusions}\label{Se:Conc}
We present a novel technique for learning a distribution over sequences from a trained LM-RNN. 
The technique allows for some variation between the predictions of the RNN's internal states while still merging them, enabling extraction of a PDFA with fewer states than in the target RNN.
It can also be terminated before completing, while still maintaining guarantees of local similarity to the target. 
The technique
does not make assumptions about the target model's representation, and can be applied to any language model -- including LM-RNNs and transformers.
It also does not require a probabilistic target, and can be directly applied to recreate any WDFA.

When applied to stochastic models such as LM-RNNs, the algorithm returns PDFAs, which are a desirable model for LM-RNN extraction because they are deterministic and therefore faster and more interpretable than WFAs. We apply it to RNNs trained on data taken from small PDFAs and HMMs, evaluating the extracted PDFAs against their target LM-RNNs and comparing to extracted WFAs and n-grams. 
When the LM-RNN has been trained on a small target PDFA, the algorithm successfully reconstructs a PDFA that has identical structure to the target, and local probabilities within tolerance of the target. 
For simple languages, our method is generally the strongest of all those considered. However for natural languages $n$-grams maintain a strong advantage. Improving our method to be competitive on naturally occuring languages as well is an interesting direction for future work.

\section*{Acknowledgments}
 The authors wish to thank R\'{e}mi Eyraud for his helpful discussions and comments, and Chris Hammerschmidt for his assistance in obtaining the results with \flexfringe.
 The research leading to the results presented in this paper is supported by the 
 Israeli Science Foundation (grant No.1319/16), and the
 European Research Council (ERC) under the European Union’s 
 Seventh Framework Programme (FP7-2007-2013),
under grant agreement no. 802774 (iEXTRACT).

\bibliographystyle{plain}
\bibliography{bib}

\clearpage
\appendix

\section*{Supplementary Material}
\section{Guarantees}\label{App:Guarantees}
We show that our algorithm returns a PDFA, and discuss the relation between the obtained PDFA $\hypoaut$ and the target $T$ when anytime stopping is and isn't used.

\subsection{Probability}
\begin{theorem} The algorithm returns a PDFA. \end{theorem}

\begin{proof} 
Let $C$ be the final clustering of $P$ achieved by the method in \cref{OurAlg:cons}. By construction, the algorithm returns a finite state machine $A=\langle C,\Sigma,c(\varepsilon),\delta_Q,\delta_W,\beta \rangle$ with well defined states, initial state, transition weights and stopping weights. 
We show that this machine is deterministic and probabilistic, i.e.:
\begin{compactenum}
\item \emph{Deterministic}: for every $c\in C,\sigma\in\Sigma$, $\delta_Q(c,\sigma)$ is uniquely defined 
\item \emph{Probabilistic}: for every $c\in C,\sigma\in\Sigma$: $\delta_Q(c,\sigma)\in [0,1]$, $\beta(c)\in [0,1]$, and
$\beta(c)+\sum_{\sigma\in\Sigma}\delta_W(c,\sigma)=1$.
\end{compactenum}
\emph{Proof of (1):} By the final refinement of the clustering (Determinism II), $k_{c,\sigma}\leq 1$ and so by construction $\delta_Q(c,\sigma)$ is assigned at most one value. If, and only if, $k_{c,\sigma}<1$, then $\delta_Q(c,\sigma)$ is assigned some best available value. So $\delta_Q(c,\sigma)$ is always assigned exactly one value.

\emph{Proof of (2)}: the values of $\delta_W$ and $\beta$ are weighted averages of probabilities, and so also in $[0,1]$ themselves. They also sum to $1$ as they are averages of distributions. Formally,
for every $c\in C$:
$$\beta(c)+\sum_{\sigma\in\Sigma}\delta_W(c,\sigma) =
$$
$$ \frac{\sum_{p\in c} P_T^p(p) P_T^l(p\concat\$)}
{\sum_{p\in c} P_T^p(p)} + 
\sum_{\sigma\in\Sigma}
\frac {\sum_{p\in c} P_T^p(p) P_T^l(p\concat\sigma)}
{\sum_{p\in c} P_T^p(p)} = $$
$$ \frac{\sum_{p\in c} P_T^p(p)P_T^l(p\concat\$)}
{\sum_{p\in c} P_T^p(p)} + 
\frac {\sum_{p\in c} \sum_{\sigma\in\Sigma} P_T^p(p) P_T^l(p\concat\sigma)}
{\sum_{p\in c} P_T^p(p)} = $$
$$ \frac {\sum_{p\in c}  P_T^p(p)
 \sum_{\sigma\in\Sigma_\$} P_T^l(p\concat\sigma)}
{\sum_{p\in c} P_T^p(p)} \underset{(*)}{=} 
\frac {\sum_{p\in c}  P_T^p(p)}
{\sum_{p\in c} P_T^p(p)} = 1
$$
where $(*)$ follows from the probabilistic behaviour of $T$: $\sum_{\sigma\in\Sigma_\$} P_T^l(p\concat\sigma)=1$ for any $p\in\Sigma^*$.
\end{proof}

\subsection{Progress}

We consider extraction using noise tolerance $t$ from some target $T=\langle Q,\Sigma,q^i,\delta_Q,\delta_W^T\rangle$. 
For the observation table $O_{P,S}$ at any stage, we denote $n_{P,S}$ the size of the largest set of pairwise $t$-distinguishable rows $O_S(p),p\in P$.

Let $A$ be an automaton constructed by the algorithm, whether or not it was stopped ahead of time. Let $O_{P,S}$ be the observation table reached before making $A$, $C\subset\mathbb{P}(P)$ be the clustering of $P$ attained when building $A$ from $O_{P,S}$ (i.e., the states of $A$), and denote $A=\langle C,\Sigma,c^i,\delta_C,\delta_W^A\rangle$.
Denote $c:P\rightarrow C$ the cluster for each prefix, i.e. $p\in c(p)$ for every $p\in P$. 
In addition, for every cluster $c\in C$, denote $p_c$ the prefix $p_c\in c$ from which $\delta_W^A(c,\circ)$ was defined when building $A$.

We show that as the algorithm progresses, it defines a monotonically increasing group of sequences $W\subset\Sigma^{+\$}$ on which the target $T$ and the algorithm's automata $A$ are $t$-consistent, and that this group is $P\concat\Sigma_\$$.

\begin{lemma} \label{Lemma:PrefClosedP}
$P$ is always prefix closed.
\end{lemma}

\begin{proof}
$P$ begins as $\{\varepsilon\}$, which is prefix closed. 
Only two operations add to $P$: closedness and counterexamples. 
When adding from closedness, the new prefix added to $P$ is of the form $p\concat\sigma$ for $p\in P,\sigma\in\Sigma$ and so $P$ remains prefix closed. 
When adding from a counterexample $w$, $w$ is added along with all of its prefixes, and so $P$ remains prefix closed.
\end{proof}

\begin{lemma} \label{Lemma:RightCluster}
For every $p\in P$, $\hat{\delta_C}(c^i,p)=c(p)$, i.e. $p\in\hat{\delta_C}(c^i,p)$. 
\end{lemma}

\begin{proof}
We show this by induction on the length of $p$. For $|p|=0$ i.e. for $\varepsilon$, 
$\hatdelta_{C}(\varepsilon)=c^i$ by definition of the recursive application of $\delta_C$,
and $c^i$=$c(\varepsilon)$ by construction (in the algorithm). We assume correctness of the lemma for $|p|=n,p\in P$. Consider $p\in P$, $|p|=n+1$, denote $p=r\concat\sigma,r\in\Sigma^*,\sigma\in\Sigma$. By the prefix closedness of $P$, $r\in P$, and so by the assumption $\hatdelta_{C}(r)=c(r)$. Now by the definition of $\hatdelta_{C}$, $\hatdelta_{C}(p)=\delta_C(\hatdelta_{C}(r),\sigma)=\delta_C(c(r),\sigma)$. By the construction of $A$, $c(r)$ is defined such that $\delta_C(c(r),\sigma)= c(p\concat\sigma)$ for every $s\in c(r)$ s.t. $s\concat\sigma\in P$, and so in particular for $r\in c(r)$, as $r\concat\sigma=p\in P$). This results in $\hatdelta_{C}(p)=\delta_Q(c(r),\sigma) = c(p)$, as desired.
\end{proof}

\begin{lemma} \label{Lemma:WeightsTSim}
For every $p\in P$ and $\sigma\in\Sigma_\$$, $\delta_A(c(p),\sigma)\tapprox P^l_T(p\concat\sigma)$.
\end{lemma}

\begin{proof}
By construction of $A$, in particular by the clique requirement for the clusters of $C$, all of the prefixes $p'\in c(p)$ satisfy $\oracle_S(p')=O_S(p')\tapprox O_S(p)=\oracle_S(p)$, and in particular for $\Sigma_\$\subseteq S$: $\oracle_{\Sigma_\$}(p')\tapprox \oracle_{\Sigma_\$}(p)$ (recall that $S$ is initiated to $\Sigma_\$$ and never reduced). 
$\delta_A(c(p),\sigma)$ is defined as the weighted average of $\oracle(p'\concat\sigma)$ for each of these $p'\in c(p)$, and so it is also $t$-equal to $\oracle(p\concat\sigma)$ i.e. $P_T^l(p\concat\sigma)$, as desired.
\end{proof}

\begin{theorem} \label{Thm:ATtCons}
For every $p\in P,\sigma\in\Sigma_\$$, $A,T$ are $t$-consistent on $p\concat\sigma$.
\end{theorem}

\begin{proof}
let $u\neq\varepsilon$ be some prefix of $p\concat\sigma$. Necessarily $v=u_{:-1}$ is some prefix of $p\in P$, and so by the prefix-closedness of $P$ (\cref{Lemma:PrefClosedP}) $v\in P$. Denote $a=u_{-1}\in\Sigma_\$$. Then
$$P^l_T(u)=P^l_T(v\concat a)\tapprox \delta_A(c(v),a)=\delta_A(\hatdelta{C}(v),a)=P^l_A(u)$$
where the second and third transitions are justified for $v\in P$ by \cref{Lemma:WeightsTSim} and \cref{Lemma:RightCluster} respectively.
This for any prefix $u\neq\varepsilon$ of $p\concat\sigma$, and so by definition $A,T$ are $t$-consistent on $p\concat\sigma$ as desired.
\end{proof}

This concludes the proof that $A,T$ are always $t$-consistent on $P\concat\Sigma_\$$. We now show that the algorithm increases $P\concat\Sigma_\$$ every finite number of operations, beginning with a direct result from \cref{Thm:ATtCons}:

\begin{corollary} \label{Cor:CexGrowsP}
Every counterexample increases $P$ by at least $1$
\end{corollary}

\begin{proof}
Recall that counterexamples to proposed automata are sequences $w\in\Sigma^{+\$}$ for which $P_T^l(w)\ntapprox P_A^l(w)$, and that they are handled by adding all their strict prefixes to $P$. 
Assume by contradiction some counterexample $w\in\Sigma^{+\$}$ for which $P$ does not increase.
Then in particular $w_{:-1}\in P$, and by \cref{Thm:ATtCons}, $P_T^l(w)=P_T^l(w_{:-1}\concat w_{-1})\tapprox P_A^l(w_{:-1}\concat w_{-1})=P_A^l(w)$, a contradiction.
\end{proof}

\begin{lemma} \label{Lemm:InconsistencyCap}
Always, $|S|\leq \frac{|P|\cdot (|P|-1)}{2}+|\Sigma_\$|$. (i.e., every $O_{P,S}$ can only have had up to $\frac{|P|\cdot (|P|-1)}{2}$ inconsistencies in its making.)
\end{lemma}

\begin{proof}
$S$ is initiated to $\Sigma_\$$, so its initial size is $|\Sigma_\$|$. $S$ is increased only following inconsistencies, cases in which there exist $p_1,p_2\in P,\sigma\in\Sigma$ s.t. $p_1\neq p_2$ $O_S(p_1)\tapprox O_S(p_2)$, but $\oracle_S(p_1)\ntapprox \oracle_S(p_2)$. Once some $p_1,p_2\in P$ cause a suffix $s$ to be added to $S$, by construction of the algorithm, $O_S(p_1)\ntapprox O_S(p_2)$ for the remainder of the run (as $s\in S$ is a suffix for which $O(p_1,s)\ntapprox O(p_2,s)$). There are exactly $\frac{|P|\cdot(|P|-1)}{2}$ pairs $p_1\neq p_2\in P$ and so that is the maximum number of possible $S$ may have been increased in any run, giving the maximum size $|S|\leq \frac{|P|\cdot (|P|-1)}{2}+|\Sigma_\$|$.
\end{proof}

(Note: If the $t$-equality relation was transitive, it would be possible to obtain a linear bound in the size of $S$. However as it is not, it is possible that a separating suffix may be added to $S$ that separates $p_1$ and $p_2$ while leaving them both $t$-equal to to some other $p_3$.)

\begin{corollary}[Progress]
For as long as the algorithm runs, it strictly expands a group $\mathbb{C}\subset\Sigma^*$ of sequences on which the automata $A$ it returns is $t$-consistent with its target $T$.
\end{corollary}

\begin{proof}
From \cref{Thm:ATtCons}, $\mathbb{C}=P\times\Sigma_\$$ is a group of sequences on which $A$ is always $t$-consistent with $T$.
We show that $\mathbb{C}$ is strictly expanding as the algorithm progresses, i.e. that every finite number of operations, $P$ is increased by at least one sequence.

The algorithm can be split into 4 operations: searching for and handing an unclosed prefix or inconsistency, building (and presenting) a hypothesis PDFA, or handling a counterexample. 
We show that each one runs in finite time, and that there cannot be infinite operations without increasing $P$.

\para{Finite Runtime of the Operations}

\emph{Building $O_{P,S}$}: Finding and handling an unclosed prefix requires a pass over all $P\times\Sigma$, while comparing row values to $P$ -- all finite as $P$ is finite (rows are also finite as $S$ is bounded by $P$'s size). Similarly finding and handling inconsistencies requires a pass over rows for all $P^2\times\sigma$, also taking finite time.

\emph{Building an Automaton} requires finding a clustering of $P$ satisfying the conditions and then a straightforward mapping of the transitions between these clusters. The clustering is built by one initial clustering (DBSCAN) over the finite set $P$ and then only refinement operations (without merges). As putting each prefix in its own cluster is a solution to the conditions, a satisfying clustering will be reached in finite time. 
\emph{Counterexamples} Handling a counterexample $w$ requires  adding at most $|w|$ new rows to $O_{P,S}$. As $S$ is finite, this is a finite operation.

\para{Finite Operations between Additions to P}
Handling an unclosed prefix by construction increases $P$, and as shown in \cref{Cor:CexGrowsP}, so does handling a counterexample.
Building a hypothesis is followed by an equivalence query, after which the algorithm will either terminate or a counterexample will be returned (increasing $P$).
Finally, by \ref{Lemm:InconsistencyCap}, the number of inconsistencies between every increase of $P$ is bounded.
\end{proof}

\renewcommand\thefigure{\thesection.\arabic{figure}}
\setcounter{figure}{0}  
\section{Example}

\begin{figure} 
  \centering
  {\includegraphics[scale=0.15]{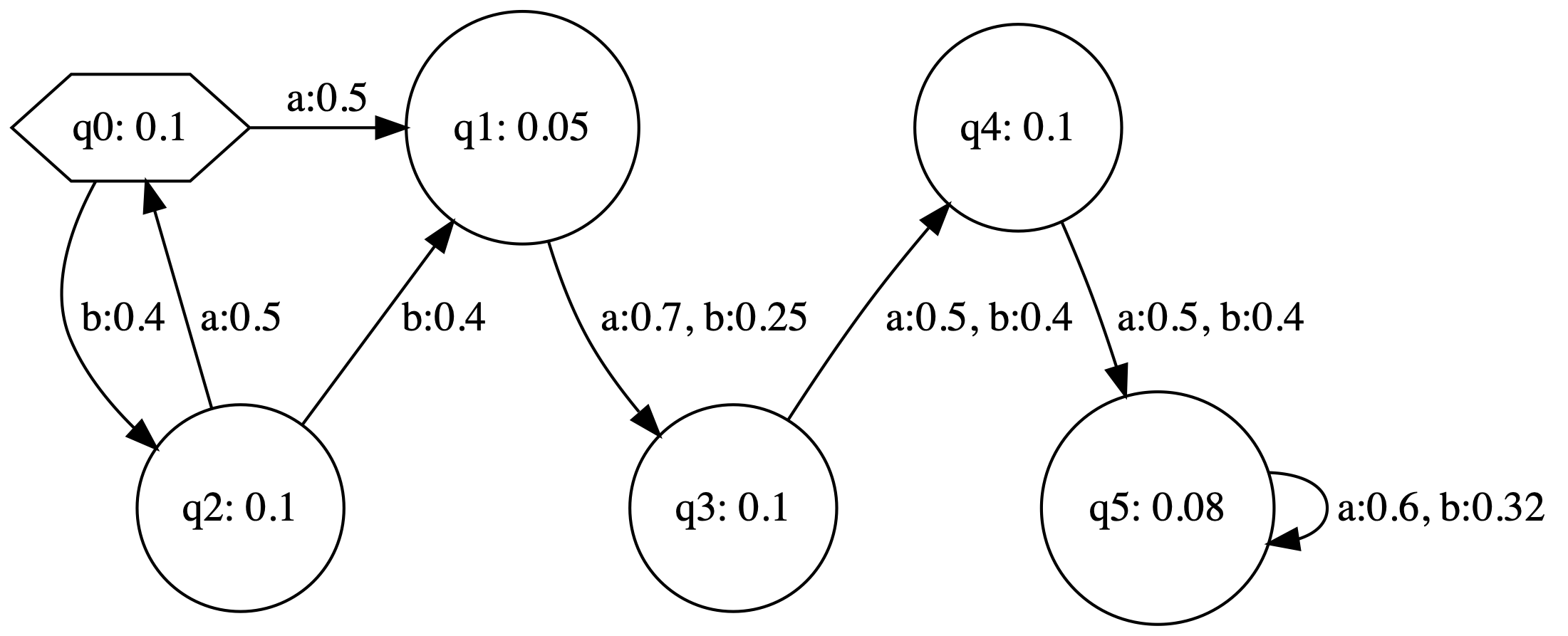}}
  \caption{Target PDFA $T$}
  \label{ex:target}
\end{figure}

\begin{figure} 
  \centering
  [$H1$]{\includegraphics[scale=0.25]{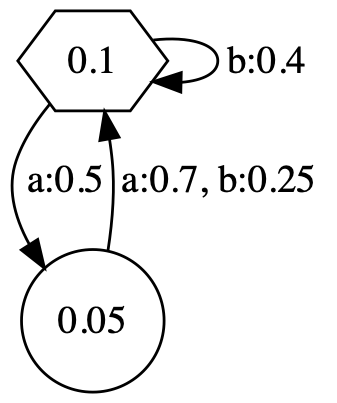}}
  [$H2$]{\includegraphics[scale=0.25]{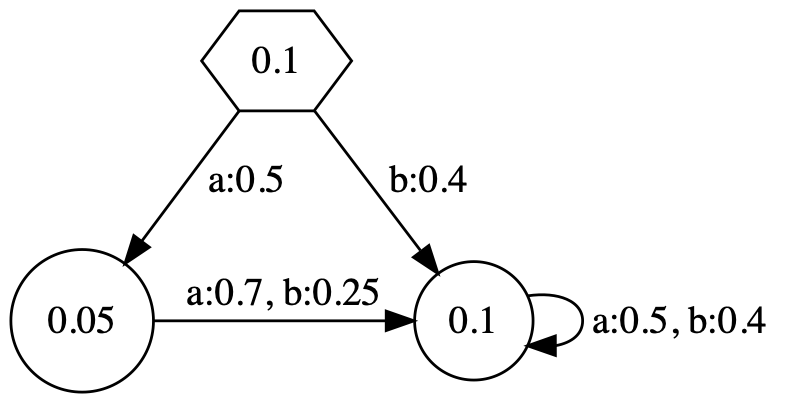}}
  [$H3$]{\includegraphics[scale=0.25]{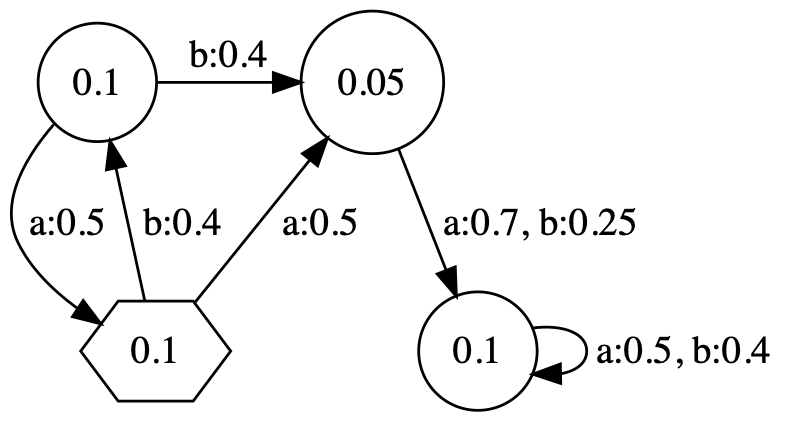}}
  \caption{Hypotheses during extraction from $T$}
  \label{ex:hypotheses}
\end{figure}

We extract from the PDFA $T$ presented in ~\ref{ex:target} using prefix and suffix thresholds $\varepsilon_P,\varepsilon_S=0$ and variation tolerance $t=0.1$. We limit the number of samples per equivalence query to $500$.
This extraction will demonstrate both types of table expansions, both types of clustering refinements, and counterexamples.
Notice that in our example, the state $q5$ is $t$-equal with respect to next-token distribution to both $q1$ and $q3$, though they themselves are not $t$-equal to each other. 

Extraction begins by initiating the table with $P=\{\emptystr\},S=\Sigma_\$$, and the queue $Q$ with $P$. We will pop from the queue in order of prefix weight, though this is not necessary when not considering anytime stopping.
At this point the table is:

 \begin{center}
 \begin{tabular}{c||  r r r }
\diagbox[width=3em]{P}{S} &  a & b & \$  \\
\hline
 $\emptystr$ & 0.5 & 0.4 & 0.1 
\end{tabular}
\end{center}

The first prefix considered is $\emptystr$, it is already in $P$. It is consistent simply as it is not similar to any other $p\in P$. However it might not be closed. Its continuations $\emptystr\concat\Sigma=\{a,b\}$ are added to $Q$, to check its closedness later. $Q$ is now $\{a,b\}$.

Next is $a$ (which has prefix weight $0.5$). $\oracle_S(a)=(0.7,0.25,0.05)$, which is not $t$-equal to the only row in the table: $O_S(\emptystr)=(0.5,0.4,0.1)$. It follows that $a_{:-1}=\emptystr$ was not closed, and $a$ is added to $P$. The table is now:

  \begin{center}
 \begin{tabular}{c||  r r r }
\diagbox[width=3em]{P}{S} &  a & b & \$  \\
\hline
 $\emptystr$ & 0.5 & 0.4 & 0.1\\
\hline
a & 0.7 & 0.25 & 0.05 
\end{tabular}
\end{center}

$a$ is also consistent simply as it has no $t$-equal rows. Its continuations $a\concat\Sigma$ are added to $Q$ to check closedness, giving $Q=\{b,ab,aa\}$. 

Now for each of $q\in Q$, $\oracle_S(q)=O_S(\emptystr)$, meaning that the table is closed. None of the prefixes in $Q$ are added to $P$, and so they are also not checked for consistency. The expansion stops and a clustering $C=\{\{\emptystr\},\{a\}\}$ is made ($\emptystr$ and $a$ are not $t$-equal). The transitions are mapped and the automaton $H1$ shown in figure \ref{ex:hypotheses} is presented for an equivalence query. 

$H1$ and $T$ are each sampled according to their distributions up to $500$ times, and $P^n_T(p),P^n_{H1}(p)$ are compared for every prefix $p$ of each sample. This soon yields the counterexample $c=aaa$, for which $P^n_{H1}(c)=(0.7,0.25,0.05)\napproxbar{0.1} (0.5,0.4,0.1)=P^n_T(c)$. $c$'s prefixes $\emptystr,a,aa,aaa$ are added to $P$ and the expansion restarts with $Q=P$ and table:

 \begin{center}
 \begin{tabular}{c||  r r r }
\diagbox[width=3em]{P}{S} &  a & b & \$  \\
\hline
 $\emptystr$ & 0.5 & 0.4 & 0.1\\
\hline
a & 0.7 & 0.25 & 0.05\\
\hline
aa & 0.5 & 0.4 & 0.1\\
\hline
aaa & 0.5 & 0.4 & 0.1 
\end{tabular}
\end{center}

$Q$ is processed: $\emptystr$ is already in $P$, $a,b$ are added to $Q$. We check its consistency with each of its $t$-equal rows, $aa$ and $aaa$, beginning with $aa$. For $a\in\Sigma$, $\oracle_S(\emptystr\concat a)=(0.7,0.25,0.05)\napproxbar{0.1}(0.5,0.4,0.1)=\oracle_S(aa\concat a)$, with the biggest difference (0.2) being on the suffix $a\in S$. The separating suffix $a\concat a\in\Sigma\concat S$ is added to $S$, separating $\emptystr$ and $aa$ in the table:

 \begin{center}
 \begin{tabular}{c||  r r r r }
\diagbox[width=3em]{P}{S} &  a & b & \$ & aa  \\
\hline
 $\emptystr$ & 0.5 & 0.4 & 0.1 & 0.7\\
\hline
a & 0.7 & 0.25 & 0.05 & 0.5\\
\hline
aa & 0.5 & 0.4 & 0.1 & 0.5\\
\hline
aaa & 0.5 & 0.4 & 0.1 & 0.6 
\end{tabular}
\end{center}

The expansion is restarted with $Q=P$. Eventually all of $P\concat\Sigma$ are processed and the table is found closed and consistent. The extraction moves to constructing a hypothesis. 

An initial clustering is made, in our case using \texttt{sklearn.cluster.DBSCAN} with parameter \texttt{min\textunderscore samples}=1. It returns $C_0=\{\{\emptystr,aa,aaa\},\{a\}\}$. However, this does not satisfy the determinism requirement: for $\emptystr$ and $aa$, which are both in the same cluster, their continuations with $a\in\Sigma$ are also in $P$ and appear in different clusters. The cluster $\{\emptystr,aa,aaa\}$ is split such that $\emptystr$ and $aa$ are separated. For $aaa$, whose continuation $aaaa$ is not in $P$, it is not important whether it joins $\emptystr$ or $aa$, and it is equally close (with respect to $L_\infty$ distance on rows) to both. The new clustering $C= \{\{aa,aaa\},\{a\},\{\emptystr\}\} $ is returned. This clustering satisfies $t$-equality ($aa\approxbar{t,S}aaa$), and a hypothesis can be made. 

For each cluster $c\in C$ there is a $p\in c$ for which $p\concat a\in P$ and so all of the $a$-transitions are simple to map. For $b$, the transitions are mapped according to the closest rows in the table, e.g. the $b$-transition from the initial state $c(\varepsilon)$ maps to $c(aa)$, as $\oracle_S(b)=(0.5,0.4,0.1,0.5)\tapprox (0.5,0.4,0.1,0.5)=O_S(aa)$. This yields the PDFA $H2$ shown in ~\ref{ex:hypotheses}.

Sampling $H2$ and $T$ soon yields the counterexample $bb$, for which $P_T^n(bb)=(0.7, 0.25, 0.05)\ntapprox (0.5, 0.4, 0.1)=P_{H2}^n(bb)$. All of $bb$'s prefixes are added to $P$, the queue is again initiated to $P$, and expansion restarts with the table:

 \begin{center}
 \begin{tabular}{c||  r r r r }
\diagbox[width=3em]{P}{S} &  a & b & \$ & aa  \\
\hline
 $\emptystr$ & 0.5 & 0.4 & 0.1 & 0.7\\
\hline
a & 0.7 & 0.25 & 0.05 & 0.5\\
\hline
aa & 0.5 & 0.4 & 0.1 & 0.5\\
\hline
aaa & 0.5 & 0.4 & 0.1 & 0.6\\
\hline
b & 0.5 & 0.4 & 0.1 & 0.5\\
\hline
bb & 0.7 & 0.25 & 0.05 & 0.5 
\end{tabular}
\end{center}

When the prefix $b$ is processed, an inconsistency is found: $b\approxbar{t,S}aa$, but $\oracle_S(bb)=(0.7,0.25,0.05,0.5)\ntapprox (0.5,0.4,0.1,0.6)=\oracle_S(aab)$, in particular on $a\in S$. $ba$ is added to $S$, $Q$ is reset to $P$, and the expansion restarts with the table:

 \begin{center}
 \begin{tabular}{c||  r r r r r }
\diagbox[width=3em]{P}{S} &  a & b & \$ & aa & ba  \\
\hline
 $\emptystr$ & 0.5 & 0.4 & 0.1 & 0.7 & 0.5\\
\hline
a & 0.7 & 0.25 & 0.05 & 0.5 & 0.5\\
\hline
aa & 0.5 & 0.4 & 0.1 & 0.5 & 0.5\\
\hline
aaa & 0.5 & 0.4 & 0.1 & 0.6 & 0.6\\
\hline
b & 0.5 & 0.4 & 0.1 & 0.5 & 0.7\\
\hline
bb & 0.7 & 0.25 & 0.05 & 0.5 & 0.5 
\end{tabular}
\end{center}

This time the table is found to be closed and consistent. \texttt{DBSCAN} gives the initial clustering $C_0= \{\{\emptystr,aa,aaa,b\},\{a,bb\}\} $, and as before the determinism refinement separates $a$ and $\emptystr$, giving $C_1= \{\{aa,aaa,b\},\{a,bb\},\{\emptystr\}\} $. Now the $t$-equality requirement is checked, and the first cluster does not satisfy it: while $aa\approxbar{t,S}aaa$ and $b\approxbar{t,S}aaa$, $aa\napproxbar{t,S}b$. The cluster is split across the suffix with the largest range, $ba$, yielding the new clustering $C= \{\{aa,aaa\},\{a,bb\},\{\emptystr\},\{b\}\} $. This clustering satisfies both determinism and $t$-equality and the hypothesis $H3$ is made, with $\sigma$-transitions from clusters $c$ for which there is no $p\in c$ such that $p\concat\sigma\in P$ (e.g. $b$ from $\{aa,aaa\}$) being made according to closest rows as described before. 

Sampling $500$ times from each of $H3$ and $T$ yields no counterexample, and indeed none exists even though the two are not exactly the same: the distributions of states $q5,q4$ and $q3$ of $T$ are $t=0.1$-equal, and the PDFAs $H3$ and $T$ are $t$-equal.

\para{A note on prefix and suffix thresholds.} Suppose that instead of $T$, we had a PDFA $T'$ over $\Sigma=\{a,b,c\}$ as follows: $T'$ is identical to $T$, except that from every state $q\in Q_T$ there is a $c$-transition with a very small probability $\varepsilon$ leading to a different state of an extremely large PDFA $L$. If $\varepsilon$ is very small, developing $L$ will be of little benefit for the approximation, but waste a lot of time and space for the extraction. However, if $\varepsilon_S,\varepsilon_P > \varepsilon$, then no prefix containing $c$ will ever be added to the table, and similarly no suffix containing $c$ will ever be considered a separating suffix (needlessly separating two prefixes). 
The existence of such transitions is quite possible in RNNs: they are unlikely to perfectly learn to represent $0$ even for tokens that have never been seen, and moreover never `tame' the states that would be reached from such transitions (as they are not seen in training).
\section{Implementation}\label{App:Implementation}

\para{Clustering the Prefixes} The initial clustering can be done with any clustering algorithm. In our implementation we use DBSCAN \cite{DBSCAN}, with $t$ as the noise tolerance and a minimum neighbourhood size $1$ for core points. When splitting a cluster into cliques, if its largest range across a single dimension is $n>1$ times the threshold $t$, it is split into $\ceil{n}$ clusters across that dimension.
In the determinism refinement, when splitting a cluster $c$, there may be some $p\in c$ for which $p\concat \sigma\notin P$. In this case a best match $c_\sigma$ for $\oracle_S(p\concat\sigma)$ is found by the heuristic given in section \ref{force-heuristics}, and $p$ is added to the respective new cluster. 

\renewcommand\thefigure{\thesection.\arabic{figure}}
\setcounter{figure}{0}  
\section{Synthetic Grammars} \label{App:smalls}

\subsection{Tomita Grammars} We adapt the Tomita grammars \cite{tomita82} for use as weighted models as follows: for each Tomita grammar and its minimal DFA $T$ we create a PDFA variant $T_W$ which has the same structure as $T$, and in which accepting/rejecting states are differentiated by their preference for \texttt{0} or \texttt{1}. Every state in $T_W$ has stopping probability $0.05$, the states $q$ have transition weights $0.7\cdot 0.95=0.665$ and $0.3\cdot 0.95=0.285$, such that $\delta_W(q,0)=0.665$ iff $q$ is an accepting state in $T$. We show all of the adaptations in \ref{fig:tomitas}, labelling the weighted variants \texttt{T1} through \texttt{T7} in the same order as their binary counterparts.
The images were generated using graphviz.

\begin{figure} 
  \centering
  [T1]{\includegraphics[scale=0.2]{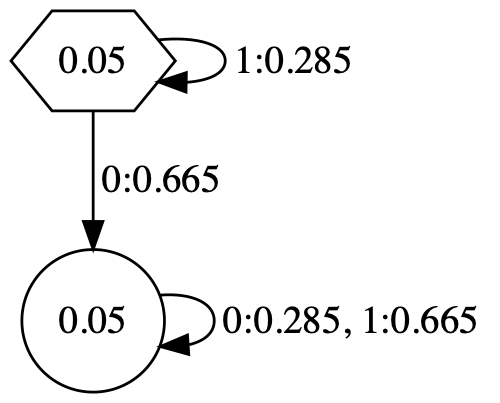}}
  [T2]{\includegraphics[scale=0.2]{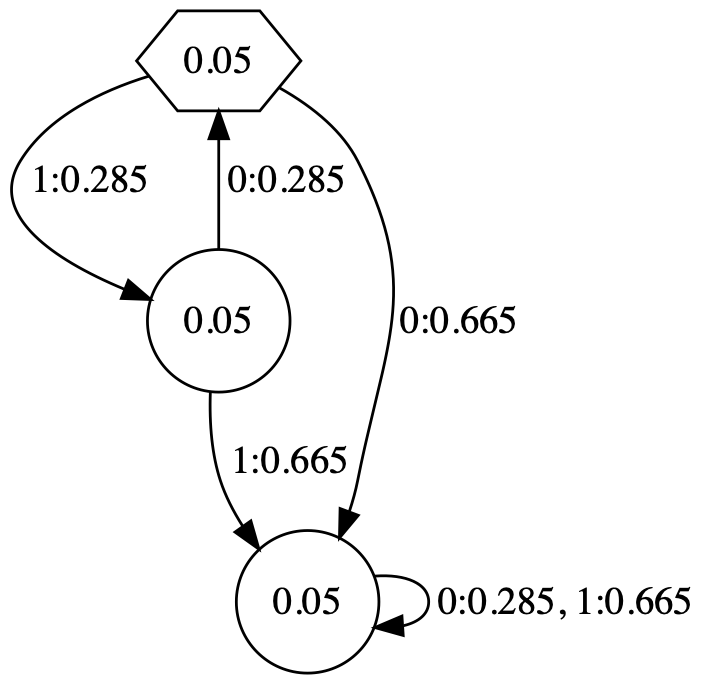}}
  [T3]{\includegraphics[scale=0.17]{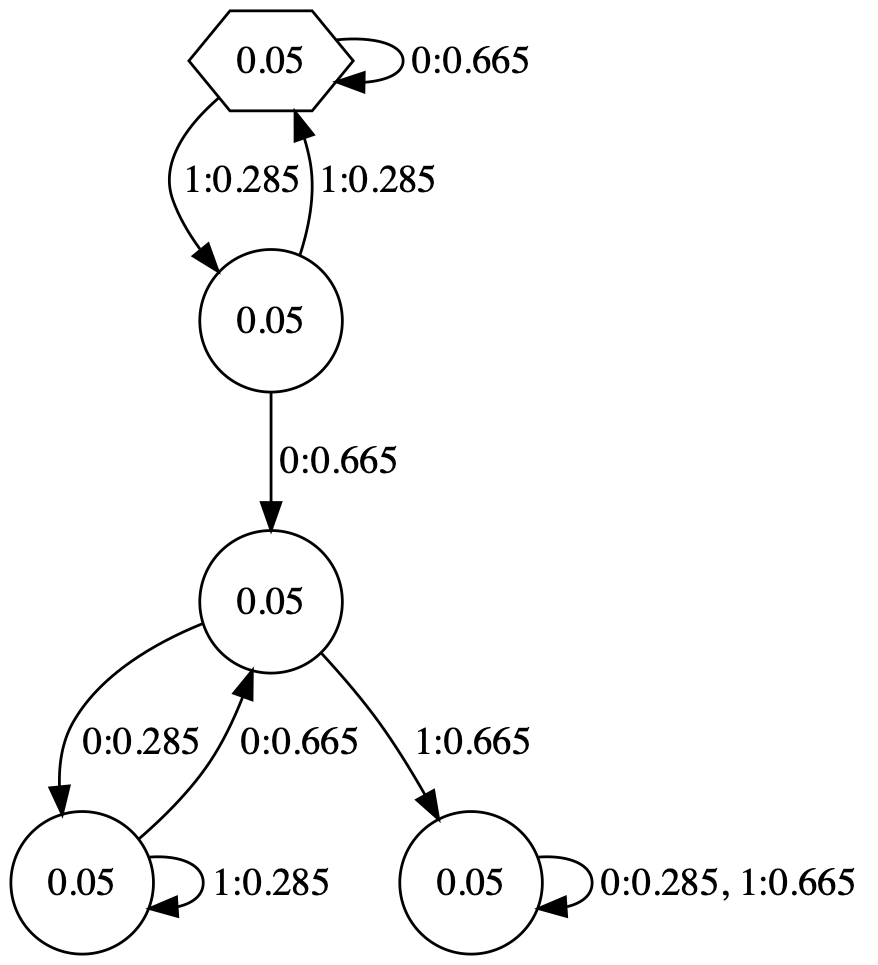}}
  [T4]{\includegraphics[scale=0.17]{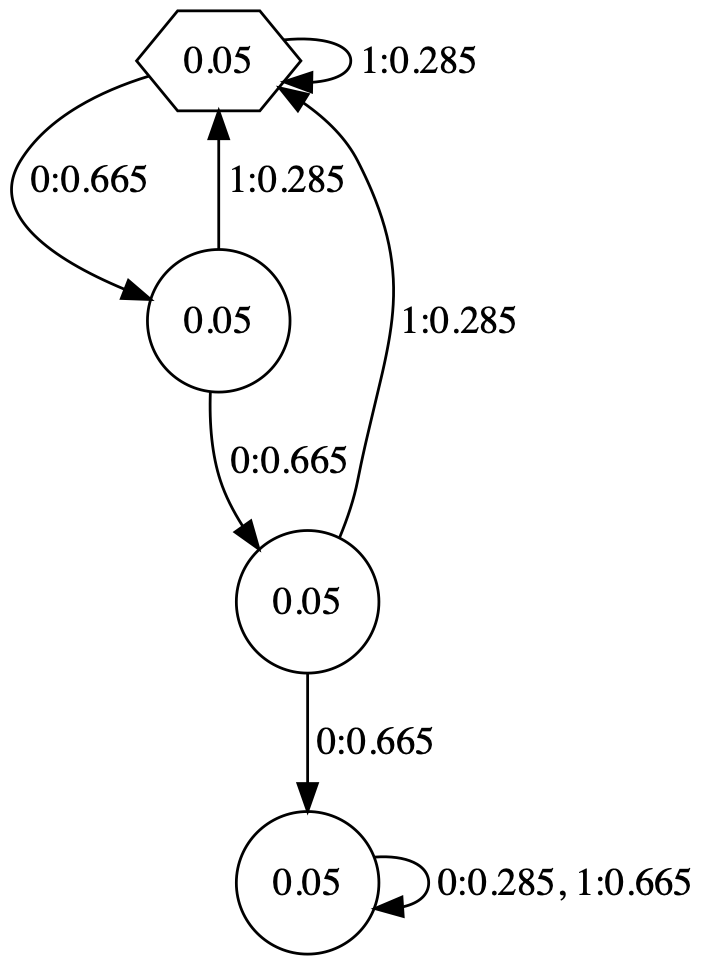}}
  [T5]{\includegraphics[scale=0.17]{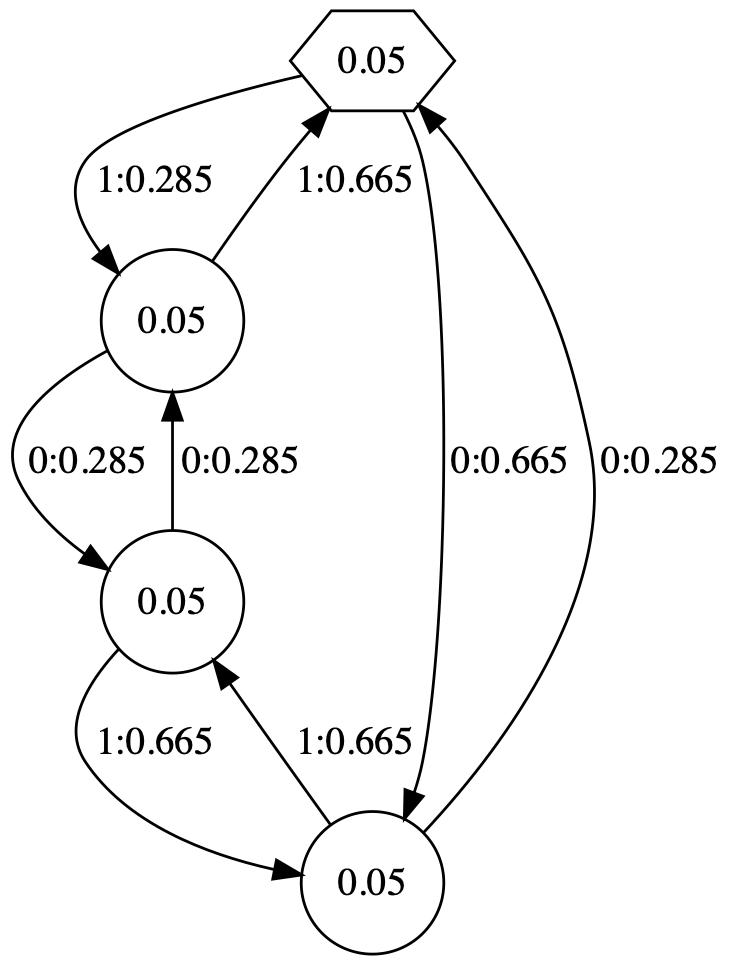}}
  [T6]{\includegraphics[scale=0.17]{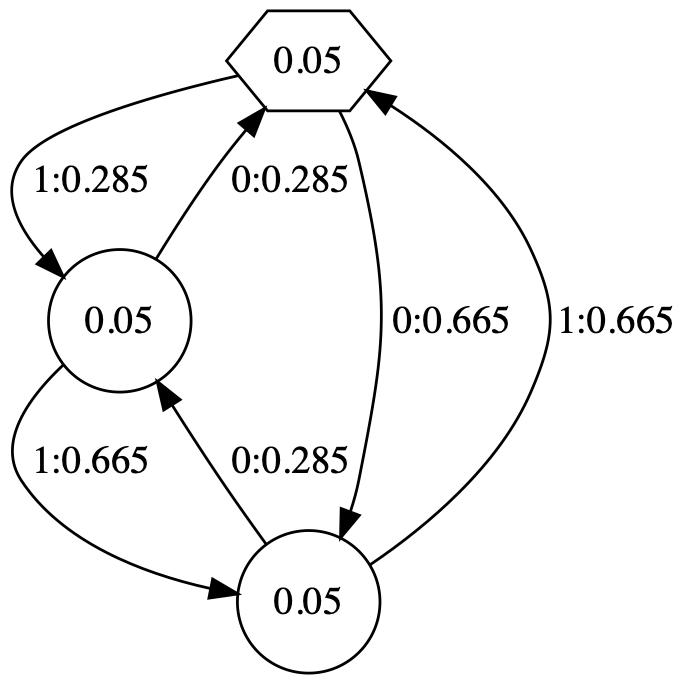}}
  [T7]{\includegraphics[scale=0.17]{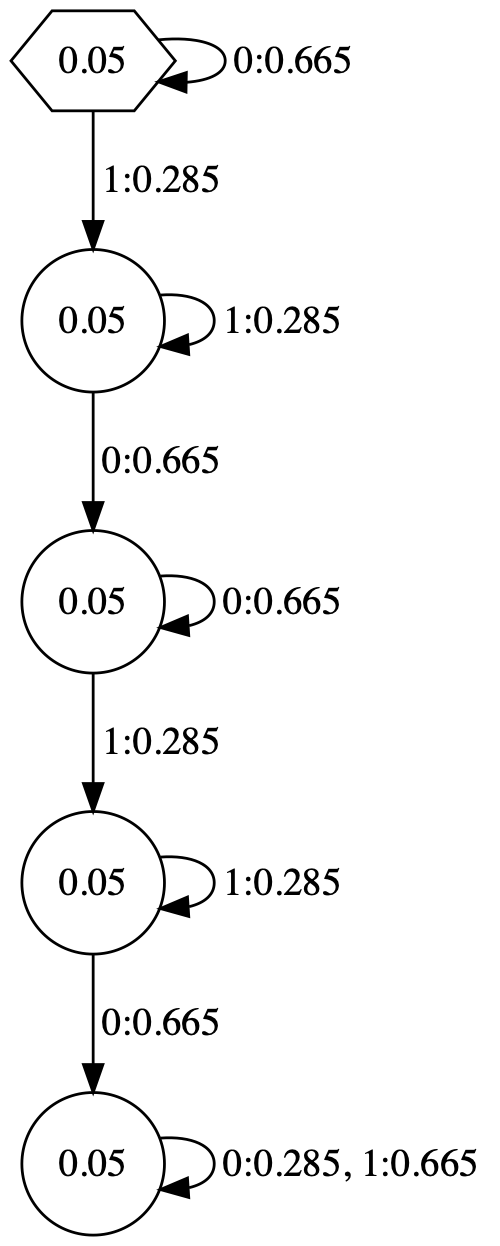}}
  \caption{Weighted variants of the Tomita grammars.}\label{fig:tomitas}
\end{figure}

We train $7$ RNNs on these grammars, their parameters and training routine are described in \ref{App:RNNs}. We extract from them with the same algorithms as for the {\spice} and {\uhl} languages. The extraction parameters and results are given in table \ref{tab:tomita}.

From each of the {\tomita} RNNs, our algorithm successfully reconstructs a PDFA with the exact same structure as the RNN's target PDFA, and transition weights within tolerance of the corresponding weights in the target. The extracted PDFAs for each {\tomita} RNN are presented in \ref{fig:extracted_tomitas}.

\begin{table}
	\centering

\small
\begin{tabular}{c l|| r r r c c}
	Language ($|\Sigma|,\ell$)& Model & WER$\downarrow$ & NDCG$\uparrow$ & Time (s) & WER Size & NDCG Size \\
	\hline

\hline
Tomita 1 ($ 2$, $ 0.77 $) & W\lstar & \textbf{0.0} & \textbf{1.0} & 55 &  2 \cexs{0} & 2 \cexs{0}  \\
 & Spectral & \textbf{0.0} & \textbf{1.0} & 18 &  k=10 & k=10  \\
 & N-Gram & 0.0001 & 0.9998 & 27 & 63 (n=6) & 31 (n=5) \\
 & ALERGIA\ignore{ nsamples/heuristic/statecount, wer: (50000, 'alergia', 50) ndcg: (50000, 'alergia', 50) }  & \textbf{0.0} & \textbf{1.0} & 28 &  8 & 8  \\
\hline
Tomita 2 ($ 2$, $ 0.78 $) & W\lstar & \textbf{0.0} & \textbf{1.0} & 55 &  3 \cexs{1}  & 3 \cexs{1}   \\
 & Spectral & \textbf{0.0} & \textbf{1.0} & 13 &  k=10 & k=10  \\
 & N-Gram & 0.0 & \textbf{1.0} & 27 & 63 (n=6) & 15 (n=4) \\
 & ALERGIA\ignore{ nsamples/heuristic/statecount, wer: (50000, 'alergia', 50) ndcg: (50000, 'alergia', 50) }  & \textbf{0.0} & \textbf{1.0} & 28 &  6 & 6  \\
\hline
Tomita 3 ($ 2$, $ 0.78 $) & W\lstar & \textbf{0.0} & \textbf{1.0} & 62 &  5 \cexs{2} & 5 \cexs{2}  \\
 & Spectral & 0.0071 & 0.9945 & 13 & k=7 & k=13 \\
 & N-Gram & 0.0542 & 0.9918 & 27 &  63 (n=6) & 63 (n=6)  \\
 & ALERGIA\ignore{ nsamples/heuristic/statecount, wer: (50000, 'alergia', 50) ndcg: (50000, 'alergia', 50) }  & 0.0318 & 0.9963 & 28 &  8 & 8  \\
\hline
Tomita 4 ($ 2$, $ 0.79 $) & W\lstar & \textbf{0.0} & \textbf{1.0} & 56 &  4 \cexs{1} & 4 \cexs{1}  \\
 & Spectral & \textbf{0.0} & \textbf{1.0} & 13 & k=14 & k=12 \\
 & N-Gram & 0.073 & 0.9887 & 27 &  63 (n=6) & 63 (n=6)  \\
 & ALERGIA\ignore{ nsamples/heuristic/statecount, wer: (50000, 'alergia', 50) ndcg: (50000, 'alergia', 50) }  & \textbf{0.0} & \textbf{1.0} & 28 &  9 & 9  \\
\hline
Tomita 5 ($ 2$, $ 0.79 $) & W\lstar & \textbf{0.0} & \textbf{1.0} & 56 &  4 \cexs{2} & 4 \cexs{2}  \\
 & Spectral & 0.0001 & \textbf{1.0} & 11 & k=67 & k=23 \\
 & N-Gram & 0.1578 & 0.9755 & 27 &  63 (n=6) & 63 (n=6)  \\
 & ALERGIA\ignore{ nsamples/heuristic/statecount, wer: (50000, 'alergia', 50) ndcg: (50000, 'alergia', 50) }  & 0.0315 & 0.991 & 29 &  15 & 15  \\
\hline
Tomita 6 ($ 2$, $ 0.78 $) & W\lstar & \textbf{0.0} & \textbf{1.0} & 56 &  3 \cexs{1} & 3 \cexs{1}  \\
 & Spectral & 0.0003 & 0.9999 & 23 &  k=36 & k=36  \\
 & N-Gram & 0.1645 & 0.9695 & 27 &  63 (n=6) & 63 (n=6)  \\
 & ALERGIA\ignore{ nsamples/heuristic/statecount, wer: (50000, 'alergia', 50) ndcg: (50000, 'alergia', 50) }  & 0.0448 & 0.9983 & 28 &  12 & 12  \\
\hline
Tomita 7 ($ 2$, $ 0.78 $) & W\lstar & \textbf{0.0} & \textbf{1.0} & 63 &  5 \cexs{1} & 5 \cexs{1}  \\
 & Spectral & 0.0003 & 0.9999 & 13 & k=32 & k=37 \\
 & N-Gram & 0.0771 & 0.9857 & 27 &  63 (n=6) & 63 (n=6)  \\
 & ALERGIA\ignore{ nsamples/heuristic/statecount, wer: (50000, 'alergia', 50) ndcg: (50000, 'alergia', 50) }  & 0.0363 & 0.9936 & 28 &  11 & 11  \\
\end{tabular}
	\caption{\tomita ~results. Each language is listed with its alphabet size $|\Sigma|$ and RNN test loss $\ell$. The $n$-grams and sample-based PDFAs were created from 50,000  samples, 			and shared samples. \flexfringe was run with state\textunderscore count =  50 . 			 Our algorithm was run with $t{=}0.1,\varepsilon_P,\varepsilon_S{=}0,|P|{\leq}5000$ and $|S|{\leq}100$,		 and spectral with $|P|,|S|{=} 100 $.}
	\label{tab:tomita}
\end{table}

\begin{figure} 
  \centering
  [ET1]{\includegraphics[scale=0.37]{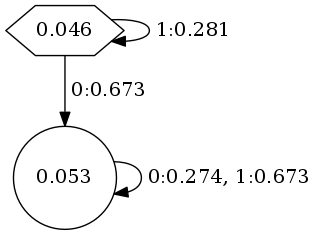}}
  [ET2]{\includegraphics[scale=0.37]{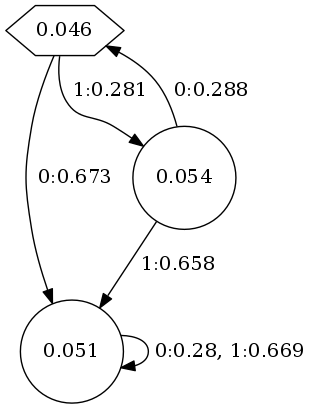}}
  [ET3]{\includegraphics[scale=0.28]{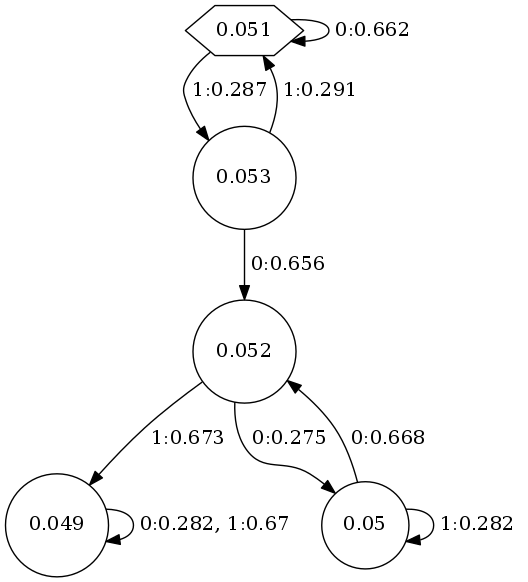}}
  [ET4]{\includegraphics[scale=0.3]{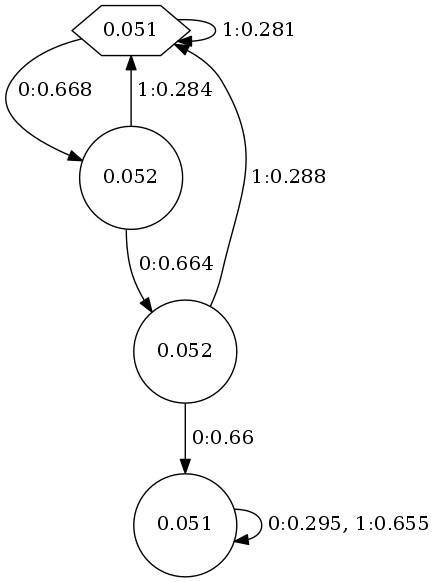}}
  [ET5]{\includegraphics[scale=0.23]{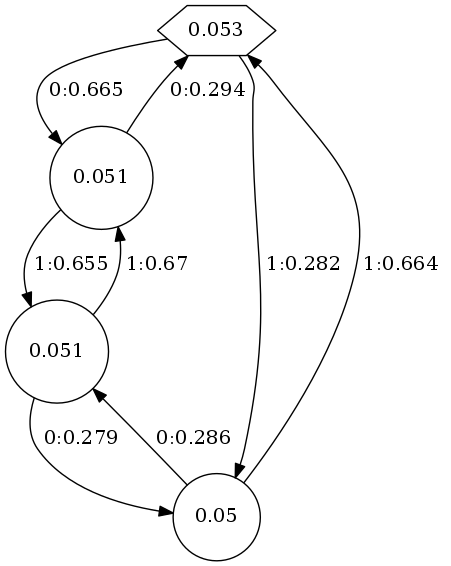}}
  [ET6]{\includegraphics[scale=0.27]{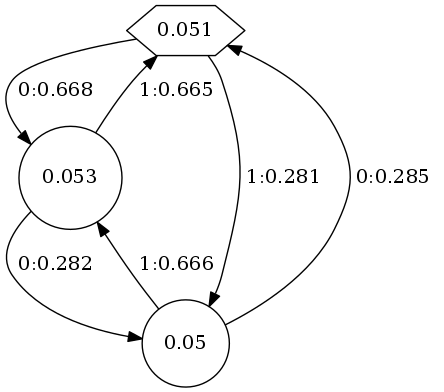}}
  [ET7]{\includegraphics[scale=0.25]{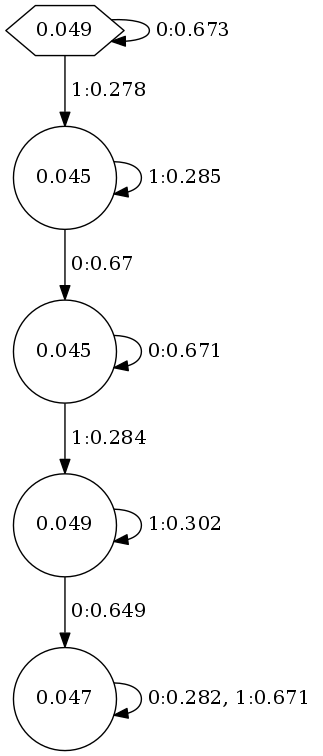}}
  \caption{PDFAs extracted using W{\lstar} from the RNNs trained on weighted variants of the Tomita grammars.}\label{fig:extracted_tomitas}
\end{figure}

\subsection{Unbounded History Languages}
The \uhl s are 3 cyclic PDFAs, shown in \ref{fig:uhls}. {\uhl} 3 is a weighted adaptation of {\tomita} 5, where the difference in probabilities between the states is lower than in our original adaptations. This makes it harder for the $n$-gram to guess the current state from local clues in its window (such as many appearances of one token over another). Precisely:

\begin{itemize}
    \item[] {\uhl }1 is a 9-state cycle PDFA over $\Sigma=\{\texttt{0},\texttt{1}\}$ that loops through all of its states one at a time, regardless of the actual input token. On all states it has stopping probability $0.05$, and divides the remaining next-token distribution over \texttt{0} and \texttt{1} as follows: on all states \texttt{0} has next-token probability $0.75$ and \texttt{1} has $0.15$, except for the second, fifth, and ninth states, where this is reversed.
    \item[] \uhl 2 is a 5-state cycle PDFA over $\Sigma=\{\texttt{0,1,2,3,4}\}$, that loops through all of its states one at a time regardless of input token. At every state it has stopping probability $0.045$, and it gives next-token probability $0.591$ to a different token at each state, with the rest of the tokens getting a uniform distribution between themselves.
    \item[] \uhl 3 is a 4-state PDFA over $\Sigma=\{\texttt{0,1}\}$ that maintains the parity of the seen \texttt{0} and \texttt{1} tokens. Every state has stopping probability $0.05$, and most states give \texttt{0} next-token probability $0.525$ and \texttt{1} next-token probability $0.425$, except for the state where the number of seen \texttt{0}s and \texttt{1}s is odd, where this is reversed. 
\end{itemize} 

\uhl 3 is an adaptation of the fifth Tomita grammar similar to our other presented adaptations, except that here the next-token probabilities of \texttt{1} and \texttt{0} are closer to each other, making it slightly harder to infer which states the PDFA has been in from a finite history\footnote{This recalls the insight of \cite{sharan16}, who note that unexpected tokens are useful as they convey information about the current state of the model.}

Applied with variation tolerance $t=0.1$, our algorithm managed to reconstruct every \uhl s structure from its trained RNN perfectly, with weights within $t$ of the original\footnote{(When extracting from RNNs, the weights of course can only be as good as those learned by the RNNs)}. The reconstructed PDFAs are shown in \ref{fig:extracted_uhls}.

\begin{figure} 
  \centering
  \begin{tabular}[c]{cc}
  \begin{tabular}[c]{c}
  [ \uhl 1]{
  \includegraphics[scale=0.15]{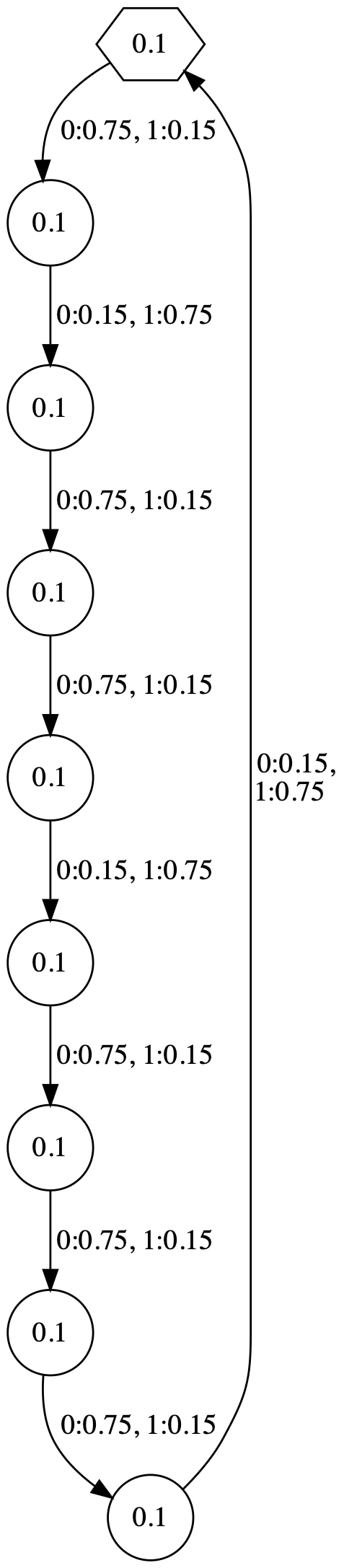}} 
  \end{tabular}
  &
  \begin{tabular}[c]{c}
  	\begin{tabular}[c]{c}
  [ \uhl 2]{\includegraphics[scale=0.15]{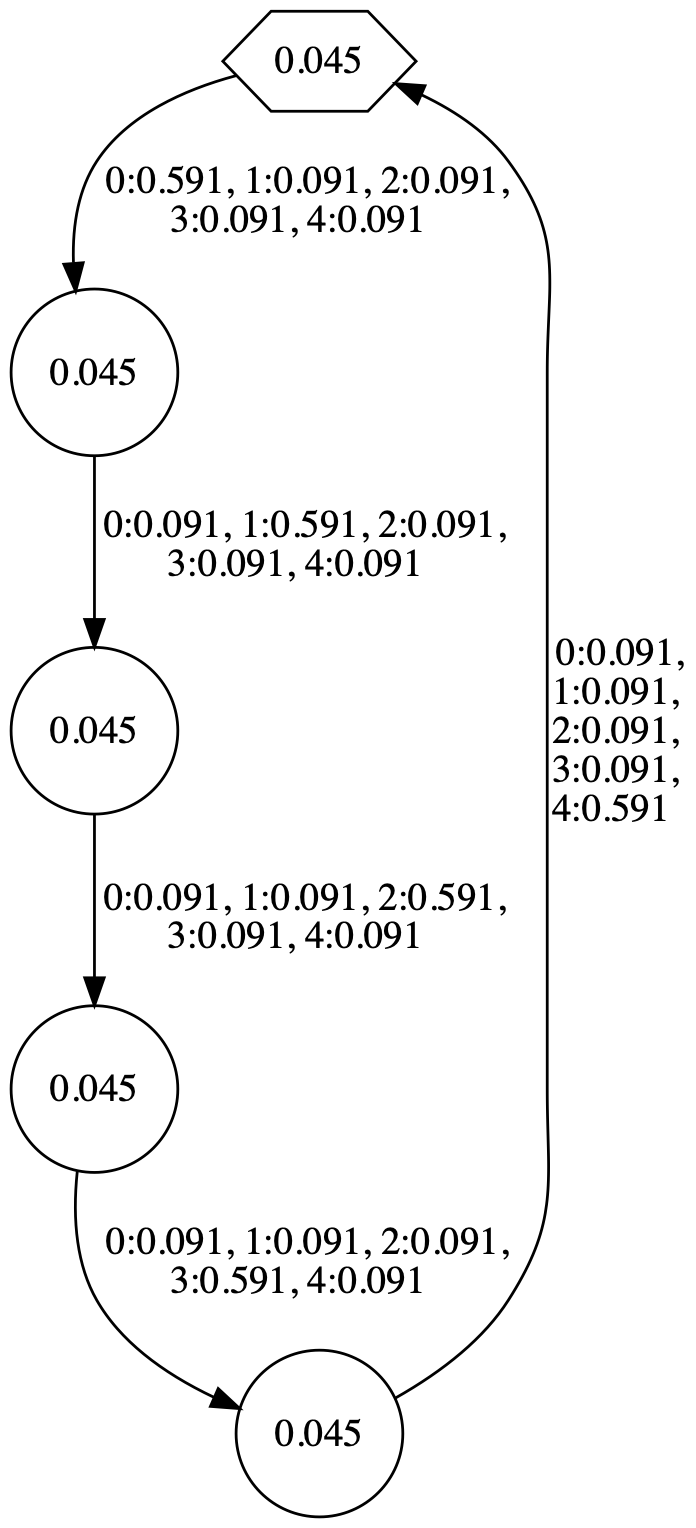}} 
  \end{tabular}  \\ 
   \begin{tabular}[c]{c}
  [ \uhl 3]{\includegraphics[scale=0.2]{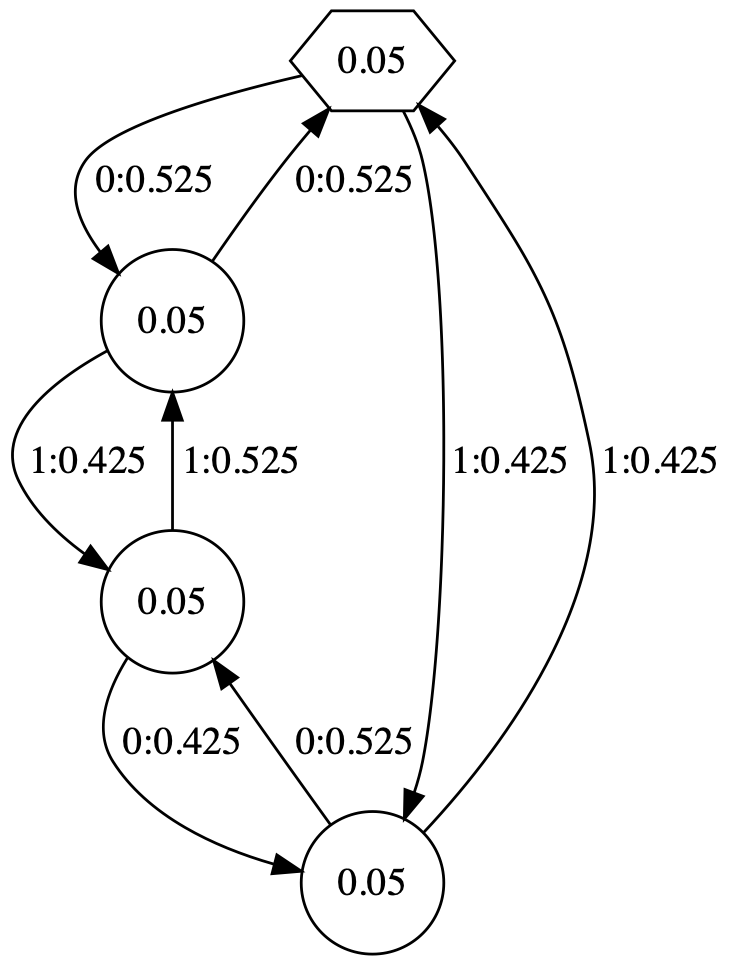}}
   \end{tabular}
\end{tabular}
  \end{tabular}
  \caption{The {\uhl} PDFAs.}\label{fig:uhls}
\end{figure}

\begin{figure} 
  \centering
  \begin{tabular}[c]{cc}
  \begin{tabular}[c]{c}
  [E\uhl 1]{
  \includegraphics[scale=0.25]{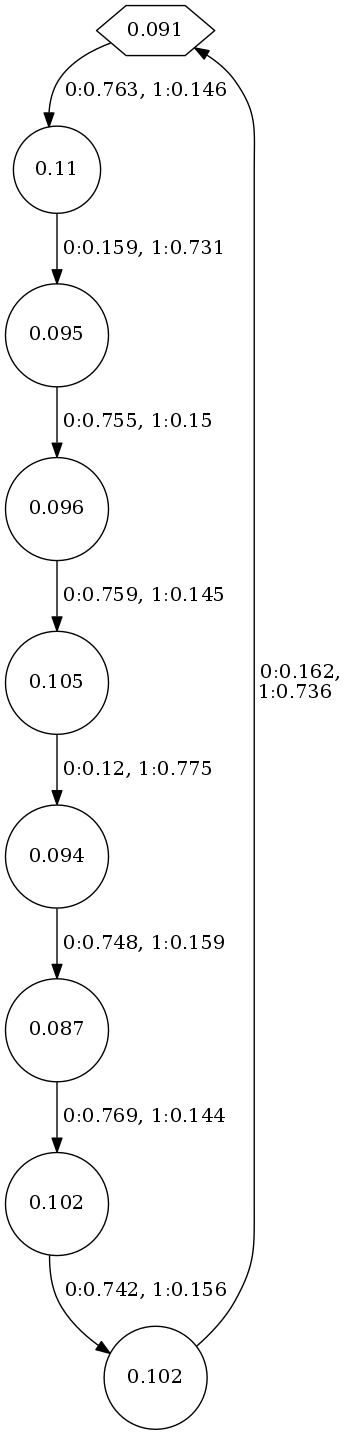}} 
  \end{tabular}
  &
  \begin{tabular}[c]{c}
\begin{tabular}[c]{c}
  [E\uhl 2]{\includegraphics[scale=0.22]{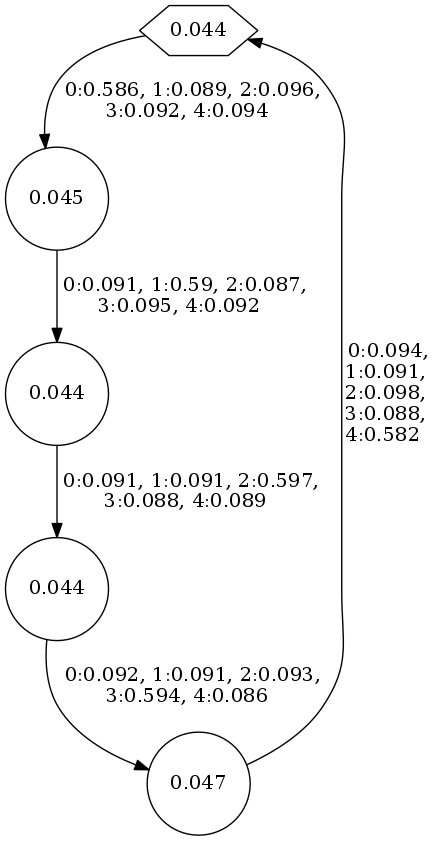}} 
     \end{tabular} \\
  \begin{tabular}[c]{c}
 [E\uhl 3]{\includegraphics[scale=0.25]{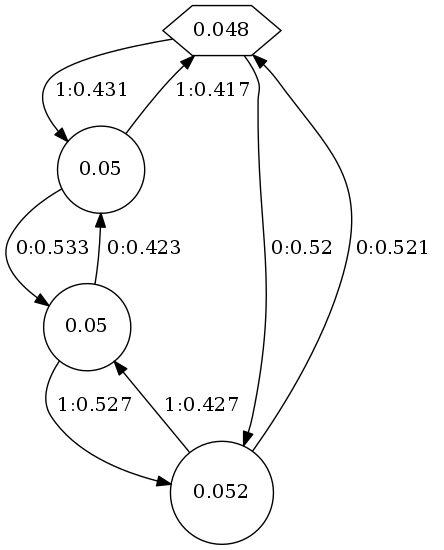}} 
    \end{tabular}
   \end{tabular}
  \end{tabular}
  \caption{The {\uhl} PDFAs, as reconstructed by W\lstar from RNNs trained on the original \uhl s.}\label{fig:extracted_uhls}
\end{figure}

\section{RNNs}\label{App:RNNs}
All the RNNs are $2$-layer pytorch LSTMs with training dropout $0.5$ and linear transformation + softmax for the classification. The input token embeddings and initial hidden states were treated as parameters. 

The {\tomita} and {\uhl} RNNs had input (embedding) dimension $2$ and hidden dimension $50$, except for {\uhl} 2 which had input dimension $5$. The {\spice} RNNs had input/hidden dimensions (resp.) as follows: $0$. 4/50 $1$. 20/50 $2$. 10/50 $3$. 10/50 $4$. 33/100 $6$. 60/100 $7$. 20/50 $9$. 11/100 $10$. 10/20 $14$. 27/30 .

The RNNs were trained with the ADAM optimiser and varying learning rates, each training for $10$ full epochs for learning rate (or less if the validation loss stopped decreasing).
The {\spice} and {\uhl} RNNs used a cyclic learning rate, going through $8$ values from $0.01$ to $0.0001$ 2 and a half times.  
The {\tomita} RNNs simply used the learning rates $0.01, 0.008, 0.006, 0.004, 0.002, 0.001, 0.0005, 0.0001, 5e-05$ once in order.

The {\spice} RNNs were trained with the train samples given by the spice competition~\cite{SPiCe}.
For the {\uhl} and {\tomita} RNNs, we generated train sets of size $10,000$ and $20,000$ respectively by sampling from the target PDFAs according to their distributions.
For each RNN, we split its given train set into train, validation, and test sets, taking respectively $90\%/5\%/5\%$ of the original set.
We checked each RNN's validation loss after every epoch. Whenever it worsened for 2 consecutive epochs, we reverted to the previous best RNN (by validation loss) and moved to the next learning rate.

For each RNN, in each training epoch we randomly split the train set into batches of equal size (up to the last `leftover' batch), and trained in these batches.
For the {\uhl} and {\tomita} RNNs we trained with batch size $500$ and for the {\spice} RNNs we used $1,000$.

\end{document}